\documentclass[twoside]{article}

\usepackage[accepted]{aistats2024}
\usepackage[utf8]{inputenc} 
\usepackage[T1]{fontenc}    
\usepackage[colorlinks=True,citecolor=blue,linkcolor=blue]{hyperref}       
\usepackage{url}            
\usepackage{booktabs}       
\usepackage{amsfonts}       
\usepackage{nicefrac}       
\usepackage{microtype}      
\usepackage{xcolor}         

\usepackage{xfrac}

\usepackage{amsopn}

\usepackage{microtype}
\usepackage{subcaption}
\usepackage{booktabs}
\usepackage{dsfont}
\usepackage{xcolor}
\usepackage{natbib}

\usepackage{algorithm}
\usepackage[noend]{algorithmic}

\usepackage[mathcal]{eucal}

\usepackage{amssymb}
\usepackage{amsfonts}

\usepackage{amsmath}
\usepackage{dsfont}

\usepackage{hyperref}

\newcommand{\esp}[1]{\mathbb{E}\left[ #1 \right]}

\newcommand{\R}{\mathbb{R}}

\newcommand{\norm}[1]{\lVert #1\rVert}

\newcommand{\cD}{\mathcal{D}}

\newcommand{\cN}{\mathcal{N}}

\newcommand{\RR}{\mathbb{R}}

\usepackage{bbm}

\newcommand{\cA}{\mathcal{A}}

\newcommand{\Id}{I_d}

\newcommand{\reg}{\mu_{\rm reg}}

\newcommand{\y}{y}
\newcommand{\rx}{\cR(x)}
\newcommand{\ry}{\cR(y)}
\newcommand{\simi}{\eta}
\newcommand{\convtodp}{\chi}
\newcommand{\Rabs}{R_{\rm abs}}
\newcommand{\Rrel}{R_{\rm rel}}
\newcommand{\cR}{\mathcal{R}}
\newcommand{\GM}{{\rm GM}_\sigma}
\newcommand{\RGM}{{\rm RGM}_{\gamma, \sigma}}
\newcommand{\Xtilde}{\tilde{X}}
\newcommand{\dproba}{\nu}

\newcommand{\sqnorm}[1]{\left\|#1 \right\|^2}
\renewcommand{\norm}[1]{\left\|#1 \right\|}

\usepackage{graphicx}

\usepackage{amsthm}

\newtheorem{definition}{Definition}
\newtheorem{proposition}{Proposition}
\newtheorem{theorem}{Theorem}
\newtheorem{lemma}{Lemma}
\newtheorem{corollary}{Corollary} 
\newtheorem{corollary*}{Corollary} 

\newcommand{\Lrel}{L_{\rm rel}}
\newcommand{\eg}{{\emph{e.g.,~}}}
\newcommand{\ie}{{\emph{i.e.,~}}}

\usepackage{natbib}

\begin{document}

\runningtitle{The Relative Gaussian Mechanism and its Application to Private Gradient Descent}

\twocolumn[

\aistatstitle{The Relative Gaussian Mechanism\\ and its Application to Private Gradient Descent}

\aistatsauthor{ Hadrien Hendrikx \And Paul Mangold \And  Aurélien Bellet }

\aistatsaddress{ Centre Inria de l'Univ. Grenoble Alpes\\
  CNRS, LJK,
  Grenoble, France \And  CMAP, UMR 7641, \\ École Polytechnique \And
  Inria, Université de Montpellier,\\France} ]

\begin{abstract}
  The Gaussian Mechanism (GM), which consists in adding Gaussian noise to a vector-valued query before releasing it, is a standard privacy protection mechanism. In particular, given that the query respects some \emph{L2 sensitivity} property (the L2 distance between outputs on any two neighboring inputs is bounded), GM guarantees Rényi Differential Privacy (RDP). Unfortunately, precisely bounding the L2 sensitivity can be hard, thus leading to loose privacy bounds. In this work, we consider a \emph{Relative} L2 sensitivity assumption, in which the bound on the distance between two query outputs may also depend on their norm. Leveraging this assumption, we introduce the \emph{Relative Gaussian Mechanism} (RGM), in which the variance of the noise depends on the norm of the output. We prove tight bounds on the RDP parameters under relative L2 sensitivity, and characterize the privacy loss incurred by using output-dependent noise. In particular, we show that RGM naturally adapts to a latent variable that would control the norm of the output. Finally, we instantiate our framework to show tight guarantees for Private Gradient Descent, a problem that naturally fits our relative L2 sensitivity assumption. 
\end{abstract}

\section{INTRODUCTION}
\label{sec:intro}

Differential Privacy (DP)~\citep{dwork2006differential} is considered the gold standard for protecting privacy, for instance in machine learning. In this framework, a curator has a database $x$, and would like to answer a query $\cR$ on $x$ by releasing an output $\cR(x)$. Yet, releasing $\cR(x)$ might reveal sensitive information on $x$. Instead, the curator may use a private algorithm $\cA$ to release a sanitized approximation $\cA(\cR)(x)$ of $\cR(x)$. To guarantee that the amount of information leaked by releasing $\cA(\cR)(x)$ is limited, DP ensures that the distributions of $\cA(\cR)(x)$ and $\cA(\cR)(y)$ are close for any $y \sim x$, \emph{i.e.}, that is close to $x$ according to a neighboring relation (databases that only differ in one row for instance).
Several divergences have been considered to measure the closeness between these two distributions, leading to different variants of DP. Among them, \emph{Rényi-Differential Privacy}~(RDP), which is based on the Rényi divergence, has become popular for its mathematical properties~\citep{mironov2017renyi}.

\begin{definition}[Rényi Differential Privacy]
  \label{def:RDP}
  A randomized algorithm $\cA$ satisfies $(\alpha, \varepsilon)$-RDP for $\alpha>1$ and $\varepsilon>0$ if $\cD_\alpha \left(\cA(x) || \cA(y) \right) \leq \varepsilon$ for all pairs of neighboring datasets $x \sim y$, 
  where $\cD_\alpha \left(\cA(x) || \cA(y) \right)$ is the $\alpha$-Rényi divergence between $\cA(x)$ and $\cA(y)$.  
\end{definition}
A fundamental building block for designing a private algorithm $\cA$ is the \emph{Gaussian Mechanism} ($\GM$)~\citep{dwork2006calibrating,dwork2014algorithmic}, which adds Gaussian noise to the private value $\cR(x)$:
\begin{equation}
    \GM(\cR)(x) = \cR(x) + \cN(0, \sigma^2), \text{ for some } \sigma^2 > 0 \,.
\end{equation}
It is very common (\eg in machine learning) to compose multiple calls to $\GM$ to build iterative algorithms like differentially private gradient descent \citep{song2013Stochastic,bassily2014Private}. 
RDP is able to tightly track the privacy guarantees of (compositions of) $\GM$, and can be converted into the more classical $(\epsilon,\delta)$-DP variant~\citep{mironov2017renyi}. 

The noise scale $\sigma^2$ of $\GM$ is based on an L2 sensitivity assumption, which guarantees that for any neighboring inputs $x \sim y$, the query $\cR$ verifies:
\begin{equation} \label{eq:l2_sensitivity}
    \sqnorm{\cR(x) - \cR(y)} \leq \Rabs^2
\end{equation}
for some $\Rabs > 0$. In particular, for $\sigma^2=\frac{\alpha \Rabs^2}{2\varepsilon}$, $\GM(\cR)(x)$ satisfies $(\alpha,\varepsilon)$-RDP. 
It is thus crucial to estimate the L2 sensitivity precisely to achieve the best possible privacy-utility trade-off. Unfortunately, this $\Rabs$ constant is often not directly known and difficult to bound tightly. In some cases, the distance between outputs is also highly correlated to the norm of these outputs, and this is the case in particular when the outputs depend on a non-private latent variable. 

Consider for instance an institute that would like to assess the mean salary for different jobs in a given company. Individual salaries are sensitive information, but people's job is not secret, and the average salary per job is the desired output. If we were to use the standard Gaussian Mechanism, then we would need an absolute sensitivity bound of the form of~\eqref{eq:l2_sensitivity} (note that other types of noises, such as Laplace, would require similar bounds in other norms, such as $L1$, but the absolute aspect would remain). To do this, the simplest approach is to use a bound on the maximum possible salary across all jobs in the company. However, this is not satisfactory since results for lower-paid jobs would be dominated by noise. An alternative is to restrict the neighboring relation to people that have the same job, which is possible since the job is not private. The problem is that estimating the salary per job (or a bound on it) is exactly what we would like to achieve in the first place. In this case, absolute sensitivity bounds are thus unsatisfactory, and would lead to unnecessarily high, as well as unfair (since the precision would be higher for well-paid jobs) estimates of the mean salary per job. Now consider that we know that by law, there should not be more than $10\%$ variations in salary for a given job in a given company: this corresponds to a \emph{relative} sensitivity assumption. In this case, one is tempted to calibrate the noise to the empirical mean salary for a given job, since we know that all the people with the same job in this company have comparable salaries. In this paper, \textbf{we tightly characterize how to scale the noise under this relative sensitivity assumption}, leading to precise and fair estimates of the mean salaries per job. 
Note that this simple example directly translates for instance to releasing gradients, where the job would be the point at which they are computed and the salary would be their magnitude. 

Our contributions are the following: \textbf{(i)} We introduce the Relative L2 Sensitivity, which generalizes the L2 sensitivity by allowing the upper bound to depend on the norm of queries. \textbf{(ii)} We leverage this assumption to introduce the \emph{Relative Gaussian mechanism} (RGM), in which the noise that we introduce depends on the output that we are about to release. \textbf{(iii)} We show tight privacy guarantees for the Relative Gaussian Mechanism. \textbf{(iv)} We show how the Relative Gaussian mechanism can be applied for Private Gradient Descent to provide adaptivity to the gradients' magnitude.

We first review related work in Section~\ref{sec:related_work}. We then define the Relative L2 Sensitivity in Section~\ref{sec:relative_sens}, and introduce RGM in Section~\ref{sec:rgm}. Finally, we instantiate the results for gradient descent on quadratics in Section~\ref{sec:gradient-descent}, and present numerical illustrations in  Section~\ref{sec:experiments}.

\section{RELATED WORK}
\label{sec:related_work}

\textbf{Local and smooth sensitivity.} Several classic techniques in the DP literature seek to avoid the calibration of noise to global sensitivity by relying on the notion of \emph{local sensitivity}. The local sensitivity $LS_\cR(x) = \max_{y : y\sim x}\|\cR(x)-\cR(y)\|$ of a dataset $x$ measures how much $\cR(y)$ can differ from $\cR(x)$ for any neighbor $y$ of $x$, which can be much smaller than the global sensitivity. In general however, calibrating the noise to the local sensitivity does not provide privacy, as two neighboring datasets may have very different local sensitivities. To go around this issue, previous work has proposed approaches based on smoothing the local sensitivity~\citep{smooth_sens}. Of particular relevance to our work,~\citet{bun2018composable} introduce truncated Concentrated Differential Privacy (tCDP), a privacy notion that is well-suited to analyzing mechanisms with smooth sensitivity. They prove 
that ``Gaussian Smooth Sensitivity'' (Gaussian mechanism used with smooth sensitivity) satisfies tCDP. This result proves to be useful in solving problems such as Gap-max, and can be extended to other noise distributions~\citep{bun2019average}. If $\cR(x)$ is of dimension $1$, the relative sensitivity assumption and the subsequent Gaussian mechanism can be seen as Gaussian Smooth Sensitivity with a special case of smooth upper bound. Yet, we go further in this paper and explore the $d$-dimensional case, which in particular allows us to consider more complex algorithms such as gradient descent. 

Other approaches include refining an absolute sensitivity bound by privately discarding outliers~\citep{tsfadia2022friendlycore}, going around the lack of global sensitivity bound by contructing a private data-dependent upper bound \citep{DBLP:conf/uai/Wang18}, or proposing and privately testing the validity of a local sensitivity bound before releasing the output~\citep{ptr}, potentially using distributional assumptions to privately estimate queries whose absolute sensitivity is unbounded \citep{brunel2020propose}. The main drawback of these approaches is that they require a special structure of the problem and release mechanism (which essentially also comes down to a certain smoothness of the local sensitivity). For gradient descent, this would for instance require an absolute bound on individual gradients and is thus not well-suited. Instead, the relative Gaussian Mechanism uses a relative sensitivity assumption which does not require the absolute boundedness of individual gradients.

\textbf{Private gradient descent.} Differentially private gradient descent (DP-GD) and its stochastic variant (DP-SGD) were first proposed by \citet{song2013Stochastic}. These algorithms and further variations have been widely studied as private minimizers of the empirical risk \citep{song2013Stochastic, bassily2014Private,wang2017Differentially}, and of the population risk \citep{bassily2019Private,feldman2020Private}. All these algorithms have been formally shown to achieve the optimal utility derived by \citet{bassily2014Private}. The analysis crucially relies on an absolute L2 sensitivity bound on the gradients (typically obtained by assuming the loss function to be Lipschitz) to calibrate the noise. Unfortunately, this often leads to the injection of excessive amounts of noise. \citet{abadi2016Deep} proposed a more practical version of DP-SGD (implemented notably in PyTorch Opacus \citep{yousefpour2021opacus} and TensorFlow Privacy \citep{abadi2016tensorflow}) which uses gradient clipping to reduce gradients' L2 sensitivity. Similarly, \citet{asi2022Private} reduced this sensitivity using a clipping-like procedure. In both cases, this decrease in L2 sensitivity introduces bias in the computation \citep{amin2019Bounding}. This phenomenon makes the analysis of clipped algorithms significantly harder \citep{chen2020Understanding,yang2022Normalized,koloskova2023Revisiting}, and it is difficult to choose a constant clipping threshold without tuning an additional hyperparameter. \citet{pichapati2019AdaCliP} and \citet{andrew2021Differentially} proposed heuristic methods for chosing clipping thresholds adaptively, although without theoretical guarantees and with limited practical applicability. Our method can reduce the amount of injected noise, while circumventing the difficulty of setting a proper clipping threshold throughout the iterations. Indeed, our relative sensitivity assumption allows the design of a relative Gaussian mechanism where noise naturally adapts to the gradients' norms. 

\section{RELATIVE L2 SENSITIVITY}
\label{sec:relative_sens}

As discussed in the introduction, we start by relaxing the restrictive L2 sensitivity assumption.
\begin{definition}[Relative L2 sensitivity] \label{def:generalized_L2} An algorithm $\cA$ satisfies Relative L2 sensitivity if there exists constants $\simi > 0$ and $\Rrel > 0$ such that for any two neighboring inputs $x \sim y$:
\begin{equation} \label{eq:rel_sensitivity}
     \sqnorm{\cR(x) - \cR(y)} \leq \simi^2 \sqnorm{\cR(x)} + \Rrel^2.
\end{equation}
\end{definition}

Note that by symmetry, this is equivalent to $\sqnorm{\cR(x) - \cR(y)} \leq \simi^2 \min(\sqnorm{\cR(x)}, \sqnorm{\cR(y)}) + \Rrel^2$. Besides, we recover the L2 sensitivity for $\simi = 0$.

\paragraph{Examples.} This definition is particularly useful when we know \emph{relative} or \emph{multiplicative} bounds on inputs. As discussed earlier, this would be the case when estimating salaries for a given job, if we know that all the people we consider have salaries within $10\%$ of each other (for instance because it is imposed by the law). We would need to know salary estimates for each job to guess the appropriate absolute sensitivity $\Rabs^2$ (or use a very imprecise global one for all jobs), whereas knowing the law directly gives us $\simi = 0.1$ and $\Rrel = 0$.

In this case, the salaries are directly correlated to a latent variable: the jobs. This is also the case for gradients, whose norm depend on the point at which they are computed. The absolute L2 sensitivity would write $\sqnorm{\nabla f(\theta) - \nabla f^\prime(\theta)} \leq \Rabs(\theta)^2$, where $f$ and $f^\prime$ are objective functions computed on neighboring datasets. Therefore, we would either need to (i) know $\Rabs(\theta)$ for all values of $\theta$, which is a lot of information, or (ii) bound it uniformly, which can be very loose. In contrast, Relative L2 sensitivity can ensure $\sqnorm{\nabla f(\theta) - \nabla f^\prime (\theta)} \leq \simi^2 \sqnorm{\nabla f(\theta)} + \Rrel^2$ with tight absolute (independent of $\theta$) parameters $\simi$ and $\Rrel$, see Section~\ref{sec:gradient-descent} for more details.

\paragraph{Links to local sensitivity.} As discussed in the related work section, the motivating idea behind local sensitivity~\citep{smooth_sens} is to set the noise according to the bound on the distance between the specific output we would like to protect and all neighboring ones. This allows much lower noise in general, since some outputs might have small sensitivity. Yet, this does not guarantee differential privacy as the level of noise injected gives information about the input that is released, as two neighboring inputs might have very different local sensitivities.

Note that Definition~\ref{def:generalized_L2} can be considered as a local sensitivity bound, since it depends on the inputs that we consider. It is stronger however: due to its symmetry, relative L2 sensitivity also ensures that two neighboring inputs also have comparable norms, and so comparable local sensitivities. This guarantees a form of smooth sensitivity~\citep{smooth_sens,bun2018composable}, and we can thus expect comparable guarantees. We will see that Definition~\ref{def:generalized_L2} allows privacy guarantees to hold even though the norm of the input is partly revealed through the noising process. In particular, we will show that Definition~\ref{def:generalized_L2} can be leveraged to release information privately even when $\Rrel \neq 0$. Our framework thus highlights another interesting example in which a form of local sensitivity can be used while still ensuring Differential Privacy, which we will leverage to analyze privatize gradient descent.

\section{THE RELATIVE GAUSSIAN MECHANISM}
\label{sec:rgm}

\subsection{Mechanism and privacy guarantees}
We now present the Relative Gaussian Mechanism ($\RGM$), and derive its privacy guarantees. $\RGM$ extends $\GM$, and leverages relative sensitivity to guarantee privacy while adapting the scale of the noise to the norm of the query. 
\begin{definition}[Relative Gaussian Mechanism]
    Let $\gamma > 0$ and $\sigma > 0$. The Relative Gaussian Mechanism of parameters $(\gamma, \sigma)$ is defined as:
    \begin{equation}
        \RGM(\cR)(x) = \cR(x) + \mathcal{N}(0, \gamma \sqnorm{\cR(x)} + \sigma^2)\,.
    \end{equation} 
\end{definition}
$\RGM$ generalizes the standard $\GM$, that we recover with $\gamma = 0$. When $\gamma > 0$, it controls to which extent $\norm{\cR(x)}$ impacts the noise. Note that $\sigma^2$ is a baseline noise, which allows to handle inputs where the query's output has small norm. For instance, if $\cR(x)=0$ on some input $x$, and $\cR(y) \neq 0$ on an input $y \sim x$, this baseline noise is necessary to guarantee privacy.

We show that, although $\RGM$ uses the query output to calibrate the noise, it can still guarantee privacy. This perhaps surprising result follows from the relative sensitivity assumption. Intuitively, this assumption ensures that all neighboring outputs have comparable norms, resulting in comparable levels of noise. 
The next theorem formalizes this intuition, deriving tight privacy guarantees for $\RGM$ on queries that satisfy a relative L2 sensitivity assumption.

\begin{theorem}[Privacy guarantees of $\RGM$] \label{thm:renyi-sgm}
Let $\cR: \cD \rightarrow \RR^d$ be a query that verifies $(\simi, \Rrel)$-relative L2 sensitivity (Definition~\ref{def:generalized_L2}) for some $\simi > 0$ and $\Rrel \ge 0$. Then for $1 \leq \alpha < (1 + \simi)^2 / (2\simi + \simi^2)$, and $\sigma^2 \geq \gamma \simi^{-2} \left[1 - \simi(\alpha - 1)\right]\Rrel^2$, $\RGM(\cR)$ satisfies $(\alpha, \epsilon)$-Rényi-DP with
\begin{equation}
    \epsilon
    =
    \frac{\alpha \simi^2}{2\gamma}\times \frac{1 + \gamma d (2 + \simi)^2(1 + \simi)^2}{1 - \simi(\alpha - 1)(2 + \simi)}.
\end{equation}
\end{theorem}
The proof is mostly technical, we thus defer it to Appendix~\ref{app:proofs}. Theorem~\ref{thm:renyi-sgm} shows that $\RGM$ can provide meaningful privacy guarantees. For a fixed $\gamma$, the guarantee is as strong as $\simi^2$ is small. This is in line with the intuition presented above: when $\simi^2$ is small, $\cR(x)$ and $\cR(y)$ (for $x \sim y$) have similar norms, and these norms are less sensitive.

\paragraph{Truncated Concentrated DP (tCDP).} Theorem~\ref{thm:renyi-sgm} can directly be turned into a $(\eta^2[\gamma^{-1} + d(2+\eta)^2(1+\eta)^2], 1 + (2\eta(2+\eta))^{-1})$-tCDP result. This is not surprising as the Relative Gaussian Mechanism can be seen as a multidimensional extension of the Gaussian Smooth Sensitivity mechanism~\citep{bun2018composable}, which offers comparable guarantees if $\cR(x) \in \R$. More details can be found in Appendix~\ref{app:gss}.

\paragraph{Scale of the noise.} The scale of the noise is controlled by the parameter $\gamma$. Indeed, for a fixed $\gamma$, our result suggests to set the baseline variance as $\sigma^2 = \gamma \simi^{-2} (1 - \simi(\alpha-1)) \Rrel^2$. As such, small values of $\gamma$ will lead to small noise addition (both in the baseline and the relative term), but will decrease privacy guarantees. Conversely, higher values of $\gamma$ require more noise for better privacy guarantees.

\paragraph{Arbitrary privacy guarantees cannot be achieved.} Although parameter $\gamma$ controls the level of the privacy guarantee, not all values of $\alpha$ and $\epsilon$ are achievable.
This is in stark contrast with the classical $\GM$ ($\simi = 0$), where increasing the noise $\sigma$ always improves privacy. This discrepancy is due to the fact that scaling noise with $\norm{\cR(x)}$ already releases some information about the input. Sadly, this information cannot be privatized using more baseline noise $\sigma^2$ without a priori bounds on $\norm{\cR(x)}$. Nonetheless, we emphasize that when $\eta \rightarrow 0$, all values of $\alpha$ and $\epsilon$ are possible.

Theorem~\ref{thm:renyi-sgm} implies that $\epsilon \geq 2\alpha\simi^2 d$, where $d$ is the dimension of the output of $\cR$. Consequently, $\RGM$ is more likely to give good privacy guarantees on small-dimensional queries. Note that this is tight, as discussed below. To mitigate this issue, one can either (i) restrict the query to a subset of its coordinates, or (ii) adapt the query to decrease the value of~$\simi$ (see discussion in Section~\ref{sec:enforcing-rel-sens}).

\paragraph{Conversion to $(\epsilon,\delta)$-DP.} 
Using Proposition~3 of \citet{mironov2017renyi}, we can convert the RDP guarantee given in Theorem~\ref{thm:renyi-sgm} to classical DP. For clarity of discussion, we give a closed-form expression of the differential privacy guarantees for the Relative Gaussian Mechanism in Corollary~\ref{cor:dp-sgm}. We stress that better guarantees can be obtained by numerically optimizing the bound obtained from Proposition~3 of \citet{mironov2017renyi}, and provide a script to choose the best values of $\alpha$ and $\gamma$ in the supplementary.
\begin{corollary}[Conversion to $(\epsilon,\delta)$-DP]
    \label{cor:dp-sgm}
    Let $0 \le \delta \le 1$. We assume that $\gamma^{-1} \ge 4 (2+\simi)^2 \log(1/\delta)$ or that $d \ge 4 \log(1/\delta) / (1 + \simi)^2$, and use the same notations as in Theorem~\ref{thm:renyi-sgm}. Then, $\RGM$ satisfies $(\epsilon, \delta)$-differential privacy with parameter $\epsilon = \convtodp + 2 \sqrt{ \convtodp \log(1/\delta)}$, where $\convtodp = \tfrac{\eta^2}{\gamma} + \eta^2(2+\eta)^2(1 + \eta)^2 d$.
\end{corollary}
We prove this result in Appendix~\ref{sup:proof-of-conv-to-dp}, but a similar one can be deduced using tCDP~\citep[Proposition 6]{bun2018composable}. While this result does not allow arbitrary privacy guarantee, we stress that meaningful guarantees can still be achieved. For instance, if $\eta=1e\text{-}3$, $d=10$, $\delta=1e\text{-}8$, and $\gamma = 100 \eta^2$, the Relative Gaussian mechanism guarantees $(\epsilon, \delta)$-DP with $\epsilon \approx 0.86$.

\paragraph{Subsampling.} We can also leverage the tCDP result to directly obtain a subsampling result~\citep[Theorem 13]{bun2018composable}. However, note that the Relative L2 sensitivity assumption should be satisfied for each batch of data. This can be demanding, in particular for gradient descent for which we will see that we will actually use a \emph{local} version of relative sensitivity. 

\subsection{Privacy Loss and Comparison with GM}
Let us consider that we use the relative sensitivity as a local sensitivity to set the noise level for disclosing output $\cA(x)$. In this case, guaranteeing $(\alpha, \varepsilon_\star)$-Rényi-DP when releasing output $\cA(x)$ requires setting the noise as $ \sigma^2_{\rm abs} = \frac{\alpha}{2\varepsilon_\star} (\simi^2 \sqnorm{\rx} + \Rrel^2)$. Unfortunately, as explained before, local sensitivity does \emph{not} guarantee differential privacy. If we were to use the same level of noise in the Relative Gaussian mechanism, this would correspond to $\gamma = \frac{\alpha \simi^2}{2\varepsilon_\star}$, and $\sigma^2 = \gamma \simi^{-2} \Rrel^2$. In particular, Theorem~\ref{thm:renyi-sgm} tells us that this choice actually guarantees RDP with parameter:
\begin{equation}
    \!\!\varepsilon =  \frac{\varepsilon_\star}{1 - \simi(\alpha - 1)(2 + \simi)} + \frac{\alpha d \simi^2 (2 + \simi)^2(1 + \simi)^2}{2(1 - \simi(\alpha - 1)(2 + \simi))}.
\end{equation}
The first term corresponds to the target privacy level $\varepsilon_\star$, weighted by a factor which is bounded by $2$ as long as $ 2 \simi(2+\simi) (\alpha - 1) \leq 1$, and goes to $1$ as $\simi$ decreases (for a fixed $\alpha$). The second term corresponds to the \emph{privacy loss} incurred by using the norm of the current output to set the noise level. Note that we see from Theorem~\ref{thm:renyi-sgm} that this term is independent of $\gamma$ and $\sigma^2$: it corresponds to a baseline loss that is paid for using a local form of sensitivity. We would get rid of this term if all possible queries $\cR(x)$ had the same norm, and this norm was public. However, this is a very strong assumption that generally does not hold (or requires very high absolute sensitivity bounds $\Rabs^2$). Note that it is tempting to use another output $\cR(y)$ to set the noise level, and thus decorrelate the noise level from the specific input that we consider. However, $\cR(y)$ would not be independent from $\cR(x)$ since $x \sim y$. 

This privacy loss term explains why using arbitrary large $\gamma$ does not lead to arbitrary good privacy guarantees. However, as long as $\simi$ is small enough compared to $\alpha$, \emph{the privacy loss is purely additive}. This means that if the dimension $d$ is not too large ($d \leq \gamma^{-1} / 36$ for $\simi \leq 1$, more for small $\simi$), we are safe using the relative Gaussian mechanism with minimal privacy overhead. Note that the $d$ term comes from the fact that we use Gaussian noise, and other noise distributions might incur other dependencies~\citep{bun2019average}. 

\paragraph{Standard vs. Relative Gaussian Mechanism.} This ``privacy loss'' point of view allows us to reason about the noise introduced by the Relative Gaussian Mechanism, versus the standard one. Indeed, let us neglect the additive privacy loss term. In this case, as argued in the previous paragraph, the privacy guarantees are comparable to the standard Gaussian mechanism with local sensitivity $\simi^2 \sqnorm{\cR(x)} + \Rrel^2$. In particular, which mechanism yields the best utility (less noise for a given privacy level) depends on which sensitivity bound is the tightest. If $\Rabs^2 \geq \simi^2 \max_{x} \sqnorm{\cR(x)} + \Rrel^2$ then the relative Gaussian mechanism is always better, because it will lead to similar guarantees with less noise overall. Otherwise, some outputs might be noised more with one mechanism and less with another. This is highly application-specific, as it is conditioned by the structure of the outputs.

\paragraph{Tightness.} One natural question that arises is the tightness of Theorem~\ref{thm:renyi-sgm}. Due to the parallel with local sensitivity, the first term is tight up to the (usually small) multiplicative factor. The second term is also tight up to a factor $1/2$ in the limit of small $\simi$, thanks to the tightness of the inequality used to obtain it. We discuss this in Appendix~\ref{app:tightness}.

\section{THE SPECIAL CASE OF GRADIENT DESCENT}
\label{sec:gradient-descent}
An important application of $\RGM$ is private Gradient Descent (GD). In this section, we describe it in the quadratic case, for which we estimate the values of $\eta$ and $\Rrel$ and propose a clipping-like procedure.

\subsection{RGM for Gradient Descent}
We consider a function $f: \RR^d \times \cD \rightarrow \RR$, where $\cD$ is a set of possible datasets. Assume that the gradients of $f$ (w.r.t. its first parameter) verify the relative sensitivity assumption. Given a dataset $D\in\cD$, we can then privately minimize function
$f$ using the following private gradient descent algorithm, where $\gamma, \sigma > 0$ are parameters of the RGM, and $\tau > 0$ is a step size: 
\begin{align} 
    \theta_{t+1} = \theta_t - \tau \RGM(\cR_{\theta_t})(D),\label{eq:gd} \\ 
    \text{ where } \cR_{\theta_t}(D) = \nabla f(\theta_t; D)\,. \nonumber 
\end{align}
We remark that the form of $\RGM$'s noise allows a tight analysis of the utility, as shown below.
\begin{theorem} \label{thm:utility} 
Let $f: \RR^d \times \cD \rightarrow \RR$ be $\mu$-strongly-convex and $L$-smooth in its first parameter~(see, \emph{e.g.}, \citet{nesterov2018lectures}). Let $D \in \cD$ be a dataset, and $\theta_\star$ be the minimizer of $f(\cdot; D)$. Assume that $f$'s gradients satisfy $(\simi, \Rrel)$-relative sensitivity, and that $\gamma, \sigma$ are set as in Theorem~\ref{thm:renyi-sgm}. Then if $\tau \leq (L + \gamma)^{-1}$, the iterates obtained by~\eqref{eq:gd} satisfy, for all $t \ge 0$,
    \begin{equation}
        \esp{\sqnorm{\theta_t - \theta_\star}} \leq (1 - \tau \mu)^t \sqnorm{\theta_0 - \theta_\star} + \frac{\tau \sigma^2}{\mu}.
    \end{equation}
\end{theorem}
The proof, along with a similar result in the general convex case are in Appendix~\ref{app:utility}. It relies on the fact that standard GD proofs already require to bound a $\sqnorm{\nabla f(\theta_t;D)}$ term by choosing an appropriate step-size, so the norm-scaled noise term can be accounted for in the same way, by only (slightly) decreasing the step size. Contrary to the usual DP-GD, which privatizes gradients using $\GM$, the variance term is $\tfrac{\tau\sigma^2}{\mu}$, where $\sigma^2$ now depends on $\Rrel$ which can be much smaller than the absolute sensitivity. In the remainder of this section, we exhibit settings in which the gradients verify relative sensitivity.
    
\subsection{Relative sensitivity for linear regression}
We now consider the specific case of quadratic objectives. More specifically, $f$ is of the form
\begin{align} 
    f(\theta; X,y) = \frac{1}{n}\sum_{i=1}^n \frac{1}{2} \sqnorm{X_i^\top \theta - y_i} + \frac{\reg}{2}\sqnorm{\theta},\label{eq:quad_objective}
\end{align} 
where $X\in \R^{d\times n}$ and $\y \in \R^n$. We denote by $(X_i, y_i) \in \R^d \times \RR$ the $i$-th data record (\ie the $i$-th column of $X$ and $i$-th element of $y$). Let $(X',y') \sim (X,y)$ be a dataset that, w.l.o.g, only differs from $(X,y)$ on its first record $(X_0, y_0)$. In the following, we denote $f = f(\cdot; D)$ and $f' = f(\cdot; D')$.

Let $\Id \in \R^{d\times d}$ be the identity matrix and let us denote $A = \frac{1}{n}XX^\top + \reg \Id \in \R^{d\times d}$, $A_i = X_i X_i^\top \in \R^{d \times d}$, $b = \frac{1}{n}X y$, and $b_i = \frac{1}{n}X_i y_i$ (and similarly for $A'$ and $b'$). For $\theta\in\RR^d$, $\nabla f(\theta) = A\theta - b$ and $\nabla f'(\theta) = A'\theta - b'$. The difference between two gradients writes:
\begin{align}
    &n^2\|\nabla f(\theta) - \nabla f^\prime(\theta)\|^2 = \sqnorm{(A_0 - A^\prime_0)\theta - b_0 + b_0^\prime}\nonumber\\
    &= \sqnorm{(A_0\! - \!A^\prime_0)A^{-1}(A\theta - b) + \!(A_0\! - A^\prime_0)A^{-1} b  - b_0\! + b_0^\prime}\nonumber\\ 
    &\leq 3\left[\sqnorm{A_0 A^{-1}} + \sqnorm{A^\prime_0 A^{-1}} \right]\sqnorm{\nabla f(\theta)}  \label{eq:sim_decomp}\\
    &\qquad + 3\sqnorm{\nabla f_0(A^{-1}b) - \nabla f_0^\prime(A^{-1}b)}, \nonumber
\end{align}
with $\norm{A} = \lambda_{\max}(A)$ being the $2$-norm for matrices. This bound hints at relative sensitivity, and we now discuss the corresponding $\simi$ and $\Rrel$ terms. We first define $L,\mu > 0$ the bounds on the largest and smallest eigenvalues of all $A$, \emph{i.e.}, $L \geq \norm{A}$, and $\mu \leq \lambda_{\min}(A)$. Then, denote $\kappa = L/\mu$.

\paragraph{General functions.} All derivations above consider quadratic objectives. Yet, similar terms with corresponding intuitions can be derived for arbitrary convex functions, as presented in Appendix~\ref{app:general_functions}.

\paragraph{The absolute term $\boldsymbol{\Rrel}$.} By using the relative framework, we only have to bound the difference between gradients at $A^{-1} b$ rather than all points, the rest being handled by the norm scaling. When an approximation of $A^{-1}b$ is known, this gives much tighter guarantees.
Otherwise, this term writes: $\norm{(A_0 - A_0^\prime)A^{-1}b - b_0 + b_0^\prime} \leq \norm{(A_0 - A_0^\prime)A^{-1}}\norm{b} + \norm{b_0 - b_0^\prime}$. In the end, one only needs to control the norm of $b$ and $b_0 - b_0^\prime$, which can be done via clipping. 

\paragraph{The relative term $\boldsymbol{\simi}$.} The term in front of the gradient norm can be bounded as: 
\begin{equation} \label{eq:condition_Lrel}
    \!\!\!\!\sqnorm{A_0 A^{-1}} \!=\! \sqnorm{X_0} X_0^\top A^{-2} X_0 \leq \kappa (X_0^\top A^{-1} X_0)^2\!.
\end{equation}  
While a direct bounding writes $\sqnorm{A_0 A^{-1}} \leq \max_i \|A_i\|^2 / \mu^2$, we will see that for well-behaved distributions, the bound only depends linearly on $\kappa$, or even not at all. This is the case for instance when the data is orthogonal, as shown in the following example.

\emph{Orthogonal data.} Relative sensitivity can be easily bounded for orthogonal data, \emph{i.e.} if either $X_i^\top X_j = \sqnorm{X_i}$ or $X_i^\top X_j = 0$. Consider that at least half of the dataset is fixed, and contains all different $X_i$ in equal proportions (so, $d^{-1}$). In this case, $A \succcurlyeq \frac{1}{2d} \sum_{i=1}^d X_i X_i^\top$ so $\sqnorm{X_i}X_i^\top A^{-2} X_i \leq 2d$. Note that the relative sensitivity is independent of the scale of each $X_i$. See Appendix~\ref{app:orthogonal_data} for more detailed derivations.

However, one can remark that we have made an extra assumption on top of data orthogonality, which is that half of the dataset is assumed to be fixed. This is because if in the dataset $X$ one dimension only has a single data point and this point is removed, then $A$ will not be invertible anymore. This could be fixed through regularization but would lose independence from $\kappa$. While this is unavoidable in the worst case,  removing a point only mildly affects the covariance matrix $A$ when the data is ``well-behaved'', leading to small $\eta$. Fortunately, there exists a mechanism designed specifically for leveraging such properties, which is Propose-Test-Release~\citep{ptr}. We will see in the next section how to instantiate it in our case. 

\subsection{Enforcing relative L2 sensitivity}
\label{sec:enforcing-rel-sens}

Sensitivity bounds are often hard to evaluate, and it is generally desirable to enforce them in practical applications, for instance using gradient clipping to restrict the absolute sensitivity.
In the quadratic case, the absolute term $\Rrel$ can be controlled by clipping the $b_i$'s. Yet, $\norm{(A_0 - A^\prime_0)A^{-1}}$  should also be controlled, hopefully without a too strong dependence on the conditioning of $A$. This can be done by clipping the $X_i$, \emph{i.e.}, shrinking their norm when it is too large. 
\begin{proposition}[Clipping] \label{prop:clipping}
    Let $C \in \R^{d\times d}$, $R_c \in \R$, and let $\tilde{X}$ be the clipped dataset, obtained as $\tilde{X}_i = R_c X_i / \max(R_c, (\sqnorm{X_i} X_i^\top C^{-2} X_i)^{\frac{1}{4}})$. If $\tilde{A} = \frac{1}{n}\sum_{i=1}^n \tilde{X}_i\tilde{X}_i^\top + \reg \Id \succcurlyeq \rho C$ for some $\rho > 0$, then $\tilde{X}$ verifies relative sensitivity with constant $\simi^2 = \frac{6 R_c^4}{\rho^2 n^2}$.
\end{proposition}
The proof directly follows from~\eqref{eq:sim_decomp},~\eqref{eq:condition_Lrel} and $\tilde{A}  \succcurlyeq \rho C$. Note that to guarantee cancellations as in~\eqref{eq:sim_decomp}, how $X_i$ is clipped should be independent of $X$ (and thus $A$), which is why we introduce matrix $C$. This result shows that given a dataset, we can enforce relative sensitivity bounds given sufficient clipping. Interestingly, using $C = \Id$ will just enforce a bound of $\tilde{\kappa}^2$ (the conditioning of the clipped covariance), but any $C$ that better captures the structure of the data will improve this bound. Thus, even loose estimates of the covariance can be used, whether they come from rough private approximations or expert data knowledge.

There are two substantial caveats to this result: $\rho$ is unknown, and $\tilde{A} \succcurlyeq \rho C$ needs to hold regardless of the dataset, and not for the specific dataset that we consider. While these seem to be particularly strong restrictions, if we can privately check that a given value of $\rho$ works for our specific dataset, then we can still guarantee differential privacy thanks to the Propose-Test-Release (PTR) framework~\citep{ptr}. 

\begin{proposition} \label{prop:ptr}
    Let $\rho > 0$, $C\in \R^{d\times d}$. Let $\Delta$ be the number of points $\tilde{X}_i$ that need to be changed from $\tilde{X}$ to obtain $\tilde{X}^\prime$ such that $\tilde{A}^\prime \preccurlyeq \rho C$. If $\tilde{A} - \rho C \preccurlyeq 0$ then $\Delta = 0$. Otherwise, $\Delta \geq \Delta_+ = \min \{|I_R|, \sum_{i \in I_R} \Xtilde_i^\top (\tilde{A} - \rho C)^{-1} \Xtilde_i \geq n\}$.
    In particular, for any $(\varepsilon, \delta)$, PTR writes: (i) Propose a bound $\rho$. (ii) Compute $\hat{\Delta} = \Delta_+ + {\rm Lap}(\varepsilon^{-1})$. (iii) If $\hat{\Delta} \leq -\log(\delta)/\varepsilon$, return $\emptyset$, otherwise, run (multiple instances of) $\RGM$ using $\rho$ to compute the relative sensitivity parameters (which guarantees $(\varepsilon_{\rm RGM}, \delta_{\rm RGM})$-DP). This mechanism is $(\varepsilon + \varepsilon_{\rm RGM}, \delta + \delta_{\rm RGM})$-DP. 
\end{proposition}
The proof can be found in Appendix~\ref{app:proof_ptr}, and more details on PTR can be found, \eg in~\citet[Section 3.2]{vadhan2017complexity}. This proposition means that the PTR framework can be implemented ``efficiently'': instead of testing all combinations, it is enough to compute and sort the $\Xtilde_i^\top (\tilde{A} - \rho C)^{-1} \Xtilde_i$, and check how many must be summed before reaching $n$. In the end, the procedure to enforce a given $\simi$ writes:
\begin{itemize}
    \item Input arbitrary $C \in \R^{d\times d}$, $R_c > 0$. Clip all $X_i$ to obtain $\Xtilde_i$ so that $\|\Xtilde_i\|^2 \Xtilde_i^\top C^{-2} \Xtilde_i \leq R_c^4$.
    \item Propose $\rho > 0$. Use PTR to privately test whether $\tilde{A} \succcurlyeq \rho C$ (Proposition~\ref{prop:ptr}). If not, abort. 
    \item If it does, then $\tilde{A}^{-2} \preccurlyeq \rho^{-2}C^{-2}$, and so $\eta^2 = \frac{6R_c^4}{\rho^2 n^2}$ can be chosen to run DP-GD with $\RGM$.
\end{itemize}
Interestingly, our approach combines several methods intended to overcome the problems of local sensitivity. We use the PTR framework to estimate the relative sensitivity parameter $\simi$. This is possible because local sensitivity is rather smooth in $\theta$, and we can efficiently bound how far $X$ is from a high local sensitivity dataset.

\subsection{Gaussian features}

We have just shown that a procedure based on PTR allows to obtain practical differential privacy guarantees for well-behaved datasets. Yet, a major question remains: \emph{how to choose $C, R_c$ and $\rho$}? First note that $\rho$ could be obtained by a binary search (paying at most log factors), and several values of clipping thresholds $R_c$ could be tried out. A key feature here is that \textbf{we only use PTR once at the beginning of training}, and do not need to rerun it at each epoch. Thus, many guesses can be tried out with limited impact on the overall privacy cost, which should be dominated by the actual gradient descent iterations. However, we give in this section a choice of $C, R_c$ and $\rho$ that works with high probability if $X_i \sim \cN(0, \Sigma)$. 

\begin{proposition} \label{prop:ellipse_gaussian}
    Let $X \in \R^{d \times n}$ such that its columns $X_i$ are drawn \emph{i.i.d.} from $\cN(0, \Sigma)$. Let us choose $C = \Sigma$, so that $\tilde{X}_i = R_c X_i / \max(R_c, X_i^\top C^{-1} X_i)^{\frac{1}{2}})$. Let $\dproba > 0$, $n \geq 4 \log(2d/\dproba) / 9$ and $\reg = 4 \norm{\Sigma} R_c^2\sqrt{\frac{\log(2d / \dproba)}{n}}$. Then, with probability at least $1 - \dproba$, the data sampled is such that
    \begin{equation}
        \rho = \frac{1}{2d(2\pi)^{d/2}}\int_{\R^d} \min(\sqnorm{u}, R_c^2) e^{-\frac{\sqnorm{u}}{2}}{\rm d} u,
    \end{equation}
    is always accepted by the PTR procedure from Proposition~\ref{prop:ptr} as long as $n \geq - \frac{\log(\delta) R_c^2}{2\varepsilon}$, leading to $\simi^2 = \frac{6 \kappa_\Sigma R_c^4}{\rho^2n^2}$, where $\kappa_\Sigma$ is the conditioning of $\Sigma$.
\end{proposition}

The proof can be found in Appendix~\ref{app:proof_clip_gaussian}. Note that we clip to enforce the stronger condition $\Xtilde_i^\top C^{-1} \Xtilde_i \leq R_c^2$, as this clipping preserves nice properties for Gaussian data. While both kinds of clipping are valid in practice, such a simple analytic expression for $\rho$ might be harder to obtain when clipping for $\|\Xtilde_i\|^2 \Xtilde_i^\top C^{-2} \Xtilde_i \leq R_c^4$. However, we pay a $\kappa_\Sigma$ factor in the relative sensitivity constant by doing so (but recover part of it since the proposed $\rho$ is independent of $\kappa_\Sigma$). 

\section{DISTRIBUTED TRAINING UNDER LOCAL DP}
\label{sec:experiments}

\begin{figure*} 
\centering
    \includegraphics*[width=0.32\linewidth]{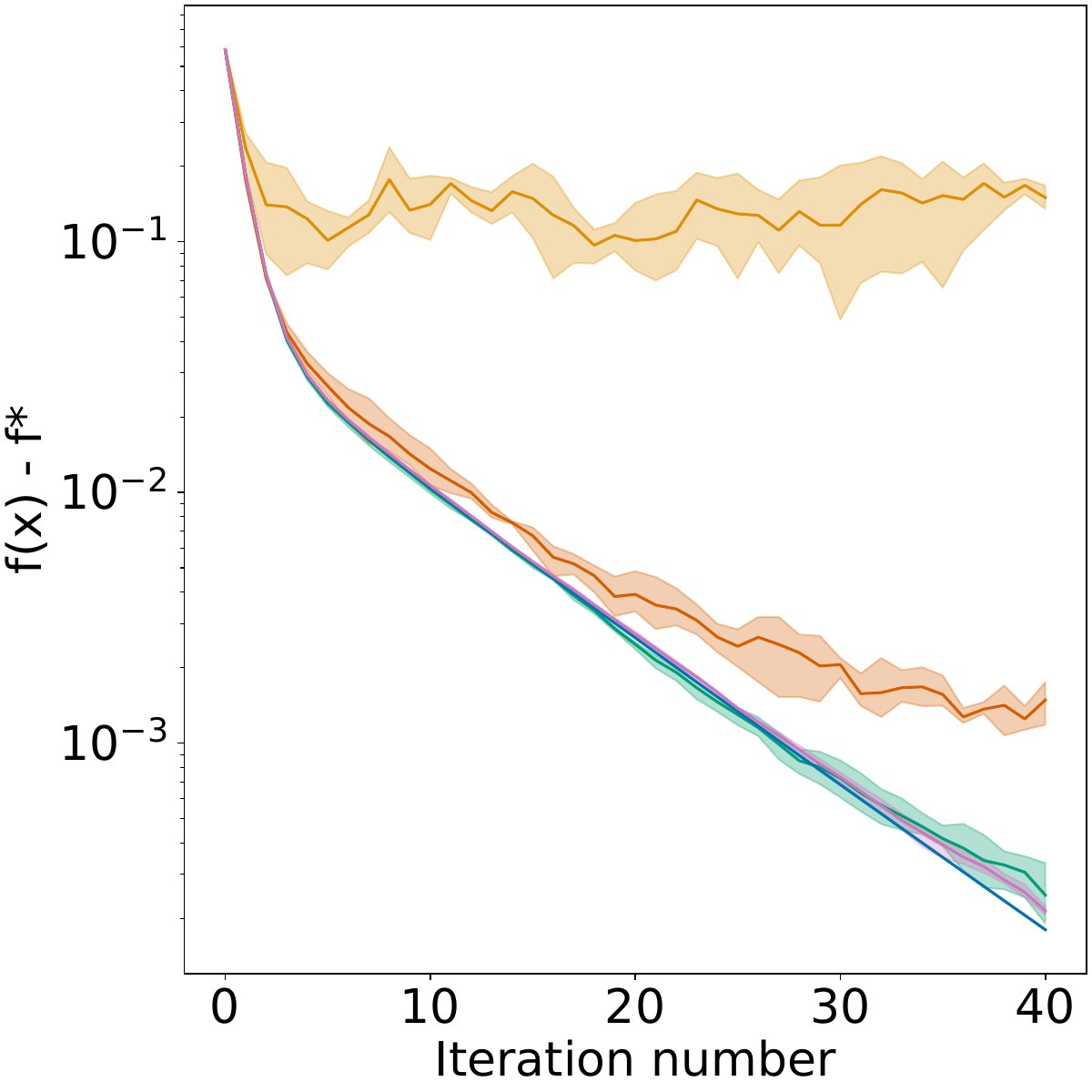}
    \includegraphics*[width=0.32\linewidth]{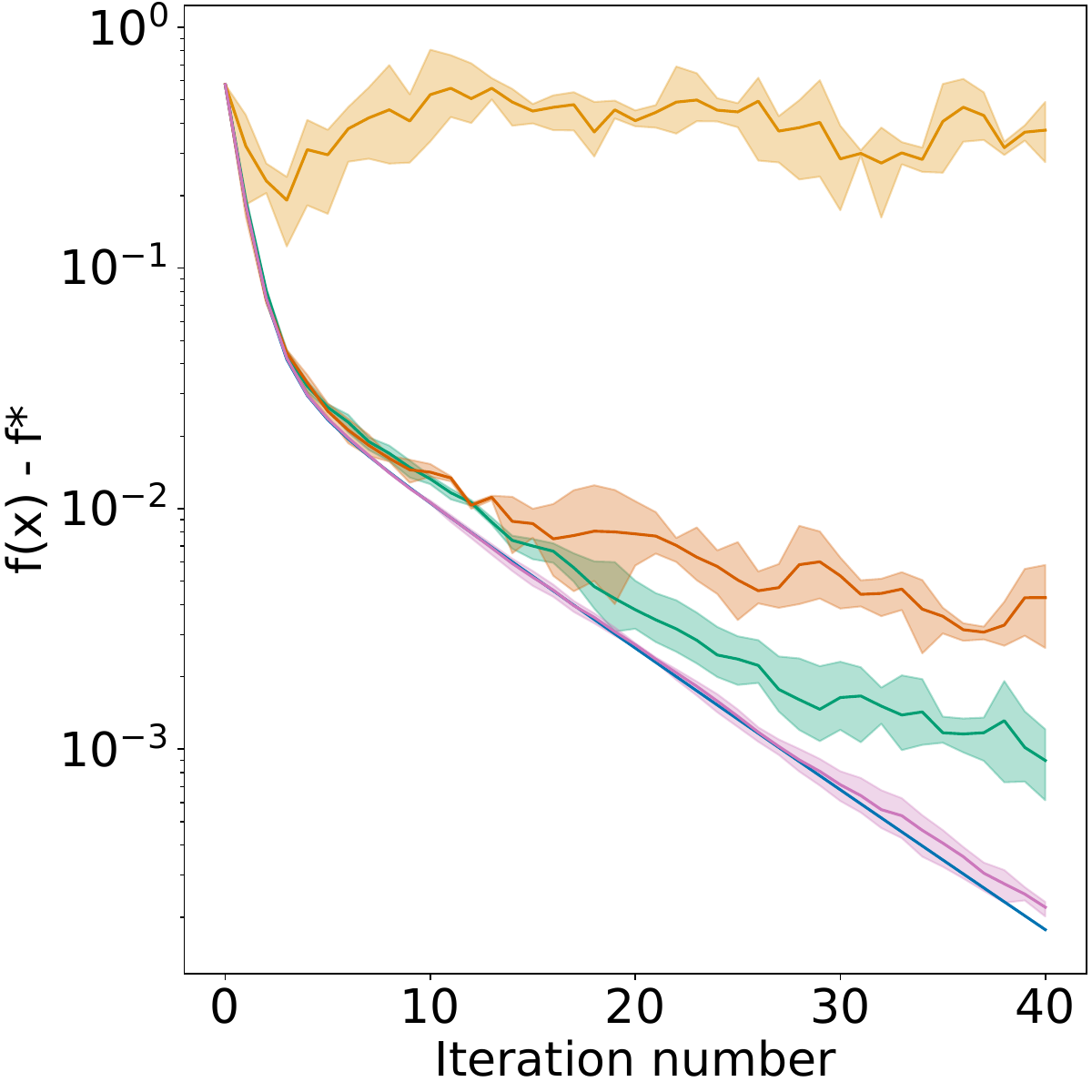}
    \includegraphics*[width=0.32\linewidth]{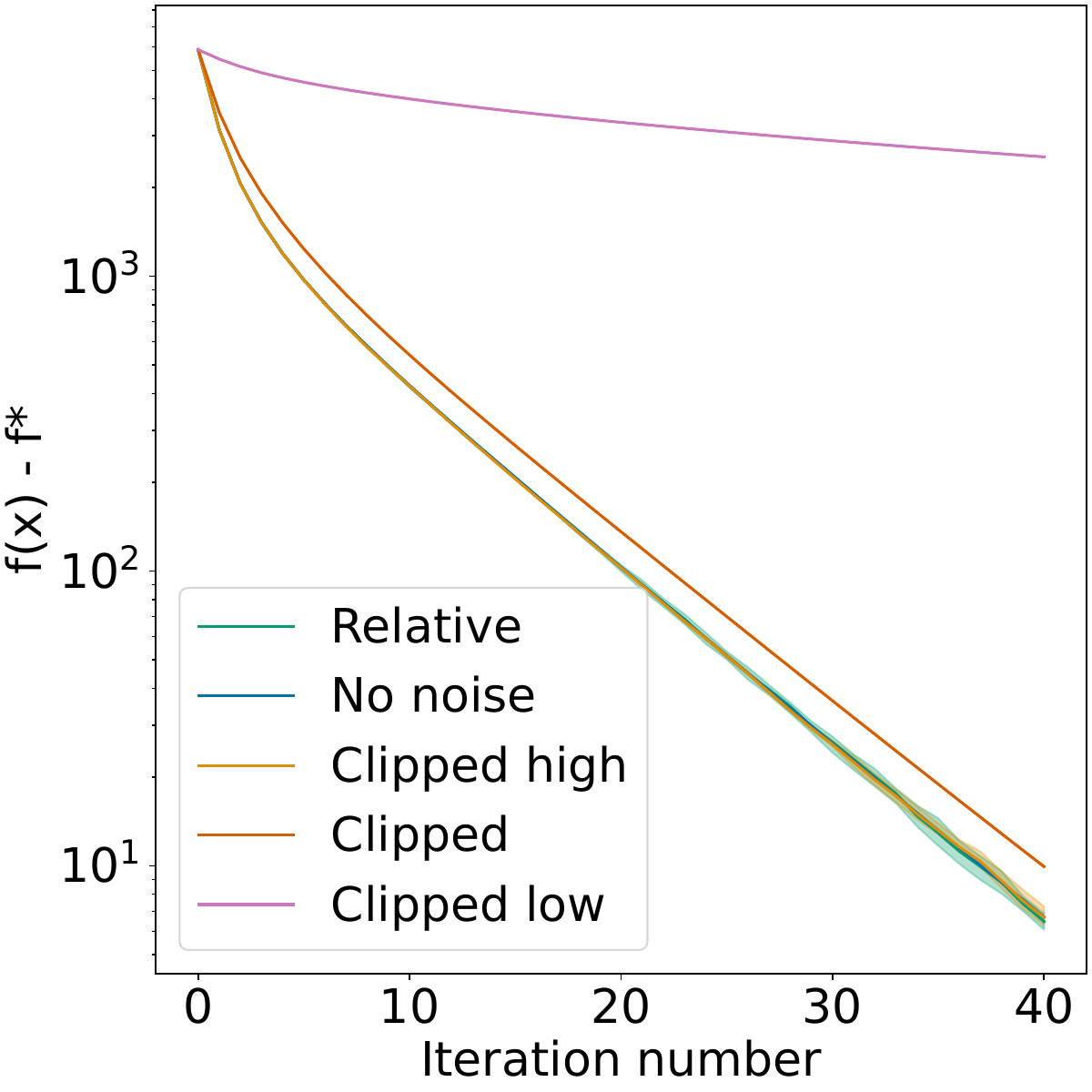}
    \caption{Utility of several private gradient descent algorithms with equivalent RDP guarantees. (Left): `Random', (Middle): `label', (Right): `bias'. Shaded areas are min/max values over 3 runs.}
    \label{fig:exp}
\end{figure*}

An interesting use-case for the Relative Gaussian Mechanism is distributed training in the local model of DP. Several nodes participate in a global training procedure, minimizing a shared objective. To this end, they periodically exchange (private) gradients, and the key bottleneck is thus \emph{communication}. Yet, privatizing gradients is a \emph{local} procedure, that each node completes on its own. In particular, it is reasonable to assume here that nodes can locally and efficiently estimate $(\simi, \Rrel)$ without having to solve the global optimization problem (which would require communication). 

We argue that, contrary to the $\RGM$, it is impossible to effectively use the $\GM$ without knowledge from other nodes. To illustrate this, consider the simple example where two nodes have respective objectives $f_1(\theta) = \alpha \sqnorm{\theta}$, and $f_2(\theta) = \beta \sqnorm{\theta - b}$. There, the sensitive records to keep private are $\alpha \in [\alpha_{\min}, \alpha_{\max}]$ and $\beta \in [\beta_{\min}, \beta_{\max}]$. Both nodes can easily compute their local parameters $\simi_1 = \alpha_{\min} / \alpha_{\max}$, $\simi_2 = \beta_{\min} / \beta_{\max}$, and $\Rrel = 0$. Then, they can directly use the $\RGM$. Consider now that nodes use $\GM$ with gradient clipping instead, and that for this particular instance, $\alpha = \beta = 1$. In order not to bias the objective, the clipping threshold for node $1$ would need to be set to a least $\norm{b}$, which is equal to the norm of the gradient of $f_1$ evaluated at the global optimum. However, node $1$ has no knowledge of $b$, and has thus no way of setting a relevant clipping threshold without exchanging information with node $2$. In this simple illustrative example, it would of course be enough to just exchange an approximation of $b$, which has a reasonable cost. Nonetheless, this highlights that setting a relevant clipping threshold in general requires knowledge of the solution to the global problem, which is generally unavailable as it depends on all nodes' data.

We illustrate this with linear regression experiments on the ijcnn1 dataset~\citep{chang2001ijcnn}, with $N=141691$, and $d=22$. We consider ridge linear regression, so $f$ is of the form of~\eqref{eq:quad_objective} where the $y_i$ correspond to the binary classification labels. We set the regularization parameter $\reg = 0.03$, and RDP parameters $\alpha = 2$ and $\varepsilon=0.1$. In order to avoid having to decide a clipping threshold for $\GM$, we automatically set the threshold as the maximum of the individual stochastic gradients at optimum (to avoid bias). We also run experiments with $c_{\rm high} = 10c$ (`Clip high') and $c_{\rm low} = c / 10$ (`Clip low'). We compare this to vanilla gradient descent without noise and $\RGM$, where we set $R_c$ just high enough so that no point is actually clipped, and take $C=\tilde{A}$ and $\rho = 1/2$ for PTR, so that the condition on $\Delta$ can be reduced to $\Delta = n/2R_c \approx 780$, leading to $\delta = \exp(-\varepsilon\Delta) \approx 10^{-34}$. The results are shown in Figure~\ref{fig:exp}. Code is available in supplementary material, and the precise experimental details can be found in Appendix~\ref{app:experiments}. Additional results on a different dataset can be found in Appendix~\ref{app:extra_exp}.

We study 3 different data splits: (i) `Random' (left plot): the data is split randomly across the two nodes. (ii) `label' heterogeneity (center): we sample 50 points at random for each node, and then all positive labels are assigned to one node and all negative to the other. (iii) `bias' (right): we add a bias $B$ to the objective to recreate (with more complex data) the simple example discussed above. 

We observe that, although there is generally always a clipping threshold that works well, this threshold is problem-dependent. Small thresholds work well for homogeneous objectives or with label heterogeneity, whereas larger clipping thresholds handle bias heterogeneity better. Therefore, the clipping threshold needs to be tuned, which requires more communication, and incurs additional privacy leaks.
On the contrary, the Relative Gaussian mechanism is, in this case, able to deal with heterogeneity without such tuning.

Note that the main contribution of the paper is to give a series of theoretical results that allow to properly instantiate smooth-sensitivity ideas, in particular for a gradient descent setting. The experiments are meant to illustrate the method and show that it can be competitive with clipping, not to show that the baseline RGM mechanism as presented in this paper gives state-of-the-art results for DP gradient descent. This would require to combine our approach with many additional tweaks (including subsampling), along with heavy tuning, and is out of the scope of this paper. Yet, our results pave the way for such applications by providing a framework that allows to evaluate relative sensitivity parameters by using PTR.

\section{CONCLUSION}
We introduced the relative L2 sensitivity, a generalization of the usual L2 sensitivity which depends on the norm of the query. We designed the Relative Gaussian mechanism ($\RGM$), a mechanism that exploits this sensitivity assumption to adapt the level of noise to the norm of the query $\cR(x)$, and proved tight privacy guarantees. We then applied $\RGM$ to private gradient descent and proposed a framework based on clipping features and PTR to enforce relative L2 sensitivity. An important research direction is to improve the mechanism to reduce the dimension dependence of the privacy guarantees. This might be achieved using relative sensitivity assumptions beyond Gaussian noise~\citep{bun2019average}.

\section*{Acknowledgements}

This work was supported by the Inria Exploratory Action
FLAMED and by the French National Research Agency
(ANR) through grant ANR-20-CE23-0015 (Project
PRIDE), ANR-20-CHIA-0001-01 (Chaire IA CaMeLOt)
and ANR 22-PECY-0002 IPOP (Interdisciplinary Project
on Privacy) project of the Cybersecurity PEPR. We thank Thomas Steinke and anonymous reviewers of a previous version of this paper for suggesting the use of the PTR mechanism to test that relative sensitivity holds.

\bibliographystyle{plainnat}
\bibliography{biblio}

\begin{thebibliography}{30}
\providecommand{\natexlab}[1]{#1}
\providecommand{\url}[1]{\texttt{#1}}
\expandafter\ifx\csname urlstyle\endcsname\relax
  \providecommand{\doi}[1]{doi: #1}\else
  \providecommand{\doi}{doi: \begingroup \urlstyle{rm}\Url}\fi

\bibitem[Abadi et~al.(2016{\natexlab{a}})Abadi, Agarwal, Barham, Brevdo, Chen,
  Citro, Corrado, Davis, Dean, Devin, et~al.]{abadi2016tensorflow}
Mart{\'\i}n Abadi, Ashish Agarwal, Paul Barham, Eugene Brevdo, Zhifeng Chen,
  Craig Citro, Greg~S Corrado, Andy Davis, Jeffrey Dean, Matthieu Devin, et~al.
\newblock Tensorflow: Large-scale machine learning on heterogeneous distributed
  systems.
\newblock \emph{arXiv preprint arXiv:1603.04467}, 2016{\natexlab{a}}.

\bibitem[Abadi et~al.(2016{\natexlab{b}})Abadi, Chu, Goodfellow, McMahan,
  Mironov, Talwar, and Zhang]{abadi2016Deep}
Martin Abadi, Andy Chu, Ian Goodfellow, H.~Brendan McMahan, Ilya Mironov, Kunal
  Talwar, and Li~Zhang.
\newblock Deep {Learning} with {Differential} {Privacy}.
\newblock In \emph{Proceedings of the 2016 {ACM} {SIGSAC} {Conference} on
  {Computer} and {Communications} {Security}}, {CCS} '16, pages 308--318, New
  York, NY, USA, October 2016{\natexlab{b}}. Association for Computing
  Machinery.
\newblock ISBN 978-1-4503-4139-4.
\newblock \doi{10.1145/2976749.2978318}.
\newblock URL \url{https://doi.org/10.1145/2976749.2978318}.

\bibitem[Amin et~al.(2019)Amin, Kulesza, Munoz, and
  Vassilvtiskii]{amin2019Bounding}
Kareem Amin, Alex Kulesza, Andres Munoz, and Sergei Vassilvtiskii.
\newblock Bounding {User} {Contributions}: {A} {Bias}-{Variance} {Trade}-off in
  {Differential} {Privacy}.
\newblock In \emph{Proceedings of the 36th {International} {Conference} on
  {Machine} {Learning}}, pages 263--271. PMLR, May 2019.
\newblock URL \url{https://proceedings.mlr.press/v97/amin19a.html}.
\newblock ISSN: 2640-3498.

\bibitem[Andrew et~al.(2021)Andrew, Thakkar, McMahan, and
  Ramaswamy]{andrew2021Differentially}
Galen Andrew, Om~Thakkar, Brendan McMahan, and Swaroop Ramaswamy.
\newblock Differentially {Private} {Learning} with {Adaptive} {Clipping}.
\newblock In \emph{Advances in {Neural} {Information} {Processing} {Systems}},
  volume~34, pages 17455--17466. Curran Associates, Inc., 2021.
\newblock URL
  \url{https://proceedings.neurips.cc/paper/2021/hash/91cff01af640a24e7f9f7a5ab407889f-Abstract.html}.

\bibitem[Asi et~al.(2022)Asi, Chadha, Cheng, and Duchi]{asi2022Private}
Hilal Asi, Karan Chadha, Gary Cheng, and John Duchi.
\newblock Private optimization in the interpolation regime: faster rates and
  hardness results.
\newblock In \emph{Proceedings of the 39th {International} {Conference} on
  {Machine} {Learning}}, pages 1025--1045. PMLR, June 2022.
\newblock URL \url{https://proceedings.mlr.press/v162/asi22a.html}.

\bibitem[Bassily et~al.(2014)Bassily, Smith, and Thakurta]{bassily2014Private}
Raef Bassily, Adam Smith, and Abhradeep Thakurta.
\newblock Private {Empirical} {Risk} {Minimization}: {Efficient} {Algorithms}
  and {Tight} {Error} {Bounds}.
\newblock In \emph{2014 {IEEE} 55th {Annual} {Symposium} on {Foundations} of
  {Computer} {Science}}, pages 464--473. IEEE, October 2014.
\newblock ISBN 978-1-4799-6517-5.
\newblock \doi{10.1109/FOCS.2014.56}.
\newblock URL
  \url{http://ieeexplore.ieee.org/lpdocs/epic03/wrapper.htm?arnumber=6979031}.

\bibitem[Bassily et~al.(2019)Bassily, Feldman, Talwar, and
  Guha~Thakurta]{bassily2019Private}
Raef Bassily, Vitaly Feldman, Kunal Talwar, and Abhradeep Guha~Thakurta.
\newblock Private {Stochastic} {Convex} {Optimization} with {Optimal} {Rates}.
\newblock In \emph{Advances in {Neural} {Information} {Processing} {Systems}},
  volume~32. Curran Associates, Inc., 2019.
\newblock URL
  \url{https://proceedings.neurips.cc/paper/2019/hash/3bd8fdb090f1f5eb66a00c84dbc5ad51-Abstract.html}.

\bibitem[Brunel and Avella-Medina(2020)]{brunel2020propose}
Victor-Emmanuel Brunel and Marco Avella-Medina.
\newblock Propose, test, release: Differentially private estimation with high
  probability.
\newblock \emph{arXiv preprint arXiv:2002.08774}, 2020.

\bibitem[Bun and Steinke(2019)]{bun2019average}
Mark Bun and Thomas Steinke.
\newblock Average-case averages: Private algorithms for smooth sensitivity and
  mean estimation.
\newblock \emph{Advances in Neural Information Processing Systems}, 32, 2019.

\bibitem[Bun et~al.(2018)Bun, Dwork, Rothblum, and Steinke]{bun2018composable}
Mark Bun, Cynthia Dwork, Guy~N Rothblum, and Thomas Steinke.
\newblock Composable and versatile privacy via truncated cdp.
\newblock In \emph{Proceedings of the 50th Annual ACM SIGACT Symposium on
  Theory of Computing}, pages 74--86, 2018.

\bibitem[Chang and Lin(2001)]{chang2001ijcnn}
Chih-Chung Chang and Chih-Jen Lin.
\newblock Ijcnn 2001 challenge: Generalization ability and text decoding.
\newblock In \emph{IJCNN'01. International Joint Conference on Neural Networks.
  Proceedings (Cat. No. 01CH37222)}, volume~2, pages 1031--1036. IEEE, 2001.

\bibitem[Chen et~al.(2020)Chen, Wu, and Hong]{chen2020Understanding}
Xiangyi Chen, Zhiwei~Steven Wu, and Mingyi Hong.
\newblock Understanding {Gradient} {Clipping} in {Private} {SGD}: {A}
  {Geometric} {Perspective}.
\newblock \emph{arXiv:2006.15429 [cs, math, stat]}, June 2020.
\newblock URL \url{http://arxiv.org/abs/2006.15429}.

\bibitem[Dwork(2006)]{dwork2006differential}
Cynthia Dwork.
\newblock Differential privacy.
\newblock In \emph{Automata, Languages and Programming: 33rd International
  Colloquium, ICALP 2006, Venice, Italy, July 10-14, 2006, Proceedings, Part II
  33}, pages 1--12. Springer, 2006.

\bibitem[Dwork and Lei(2009)]{ptr}
Cynthia Dwork and Jing Lei.
\newblock Differential privacy and robust statistics.
\newblock In \emph{STOC}, 2009.

\bibitem[Dwork et~al.(2006)Dwork, McSherry, Nissim, and
  Smith]{dwork2006calibrating}
Cynthia Dwork, Frank McSherry, Kobbi Nissim, and Adam Smith.
\newblock Calibrating noise to sensitivity in private data analysis.
\newblock In \emph{Theory of Cryptography: Third Theory of Cryptography
  Conference, TCC 2006, New York, NY, USA, March 4-7, 2006. Proceedings 3},
  pages 265--284. Springer, 2006.

\bibitem[Dwork et~al.(2014)Dwork, Roth, et~al.]{dwork2014algorithmic}
Cynthia Dwork, Aaron Roth, et~al.
\newblock The algorithmic foundations of differential privacy.
\newblock \emph{Foundations and Trends{\textregistered} in Theoretical Computer
  Science}, 9\penalty0 (3--4):\penalty0 211--407, 2014.

\bibitem[Feldman et~al.(2020)Feldman, Koren, and Talwar]{feldman2020Private}
Vitaly Feldman, Tomer Koren, and Kunal Talwar.
\newblock Private stochastic convex optimization: optimal rates in linear time.
\newblock In \emph{Proceedings of the 52nd Annual ACM SIGACT Symposium on
  Theory of Computing}, pages 439--449, 2020.

\bibitem[Koloskova et~al.(2023)Koloskova, Hendrikx, and
  Stich]{koloskova2023Revisiting}
Anastasia Koloskova, Hadrien Hendrikx, and Sebastian~U. Stich.
\newblock Revisiting {Gradient} {Clipping}: {Stochastic} bias and tight
  convergence guarantees, May 2023.
\newblock URL \url{http://arxiv.org/abs/2305.01588}.

\bibitem[Mironov(2017)]{mironov2017renyi}
Ilya Mironov.
\newblock R{\'e}nyi differential privacy.
\newblock In \emph{2017 IEEE 30th computer security foundations symposium
  (CSF)}, pages 263--275. IEEE, 2017.

\bibitem[Nesterov et~al.(2018)]{nesterov2018lectures}
Yurii Nesterov et~al.
\newblock \emph{Lectures on convex optimization}, volume 137.
\newblock Springer, 2018.

\bibitem[Nissim et~al.(2007)Nissim, Raskhodnikova, and Smith]{smooth_sens}
Kobbi Nissim, Sofya Raskhodnikova, and Adam Smith.
\newblock Smooth sensitivity and sampling in private data analysis.
\newblock In \emph{STOC}, 2007.

\bibitem[Pichapati et~al.(2019)Pichapati, Suresh, Yu, Reddi, and
  Kumar]{pichapati2019AdaCliP}
Venkatadheeraj Pichapati, Ananda~Theertha Suresh, Felix~X. Yu, Sashank~J.
  Reddi, and Sanjiv Kumar.
\newblock {AdaCliP}: {Adaptive} {Clipping} for {Private} {SGD}.
\newblock \emph{arXiv:1908.07643 [cs, stat]}, October 2019.
\newblock URL \url{http://arxiv.org/abs/1908.07643}.

\bibitem[Song et~al.(2013)Song, Chaudhuri, and Sarwate]{song2013Stochastic}
Shuang Song, Kamalika Chaudhuri, and Anand~D. Sarwate.
\newblock Stochastic gradient descent with differentially private updates.
\newblock In \emph{2013 {IEEE} {Global} {Conference} on {Signal} and
  {Information} {Processing}}, pages 245--248, Austin, TX, USA, December 2013.
  IEEE.

\bibitem[Tropp et~al.(2015)]{tropp2015introduction}
Joel~A Tropp et~al.
\newblock An introduction to matrix concentration inequalities.
\newblock \emph{Foundations and Trends{\textregistered} in Machine Learning},
  8\penalty0 (1-2):\penalty0 1--230, 2015.

\bibitem[Tsfadia et~al.(2022)Tsfadia, Cohen, Kaplan, Mansour, and
  Stemmer]{tsfadia2022friendlycore}
Eliad Tsfadia, Edith Cohen, Haim Kaplan, Yishay Mansour, and Uri Stemmer.
\newblock Friendlycore: Practical differentially private aggregation.
\newblock In \emph{International Conference on Machine Learning}, pages
  21828--21863. PMLR, 2022.

\bibitem[Vadhan(2017)]{vadhan2017complexity}
Salil Vadhan.
\newblock The complexity of differential privacy.
\newblock \emph{Tutorials on the Foundations of Cryptography: Dedicated to Oded
  Goldreich}, pages 347--450, 2017.

\bibitem[Wang et~al.(2017)Wang, Ye, and Xu]{wang2017Differentially}
Di~Wang, Minwei Ye, and Jinhui Xu.
\newblock Differentially private empirical risk minimization revisited:
  {Faster} and more general.
\newblock In I.~Guyon, U.~V. Luxburg, S.~Bengio, H.~Wallach, R.~Fergus,
  S.~Vishwanathan, and R.~Garnett, editors, \emph{Advances in neural
  information processing systems}, volume~30. Curran Associates, Inc., 2017.
\newblock URL
  \url{https://proceedings.neurips.cc/paper/2017/file/f337d999d9ad116a7b4f3d409fcc6480-Paper.pdf}.

\bibitem[Wang(2018)]{DBLP:conf/uai/Wang18}
Yu{-}Xiang Wang.
\newblock Revisiting differentially private linear regression: optimal and
  adaptive prediction {\&} estimation in unbounded domain.
\newblock In Amir Globerson and Ricardo Silva, editors, \emph{Proceedings of
  the Thirty-Fourth Conference on Uncertainty in Artificial Intelligence, {UAI}
  2018, Monterey, California, USA, August 6-10, 2018}, pages 93--103. {AUAI}
  Press, 2018.
\newblock URL \url{http://auai.org/uai2018/proceedings/papers/40.pdf}.

\bibitem[Yang et~al.(2022)Yang, Zhang, Chen, and Liu]{yang2022Normalized}
Xiaodong Yang, Huishuai Zhang, Wei Chen, and Tie-Yan Liu.
\newblock Normalized/{Clipped} {SGD} with {Perturbation} for {Differentially}
  {Private} {Non}-{Convex} {Optimization}, June 2022.
\newblock URL \url{http://arxiv.org/abs/2206.13033}.

\bibitem[Yousefpour et~al.(2021)Yousefpour, Shilov, Sablayrolles, Testuggine,
  Prasad, Malek, Nguyen, Ghosh, Bharadwaj, Zhao, et~al.]{yousefpour2021opacus}
Ashkan Yousefpour, Igor Shilov, Alexandre Sablayrolles, Davide Testuggine,
  Karthik Prasad, Mani Malek, John Nguyen, Sayan Ghosh, Akash Bharadwaj,
  Jessica Zhao, et~al.
\newblock Opacus: User-friendly differential privacy library in pytorch.
\newblock \emph{arXiv preprint arXiv:2109.12298}, 2021.

\end{thebibliography}

\newpage
\onecolumn

\appendix

\begin{center}
    {\Large\bfseries Appendix}
\end{center}

The appendix is organized as follows. Section~\ref{app:proofs} contains the full proofs for the general $\RGM$, and in particular Theorem~\ref{thm:renyi-sgm}. Section~\ref{app:gradients_quadratics} contains the proofs for the results that justify using relative assumptions when minimizing quadratics, and Section~\ref{app:experiments} contains the detailed experimental setting, with all the elements needed to reproduce the experiments. The code itself can be found in supplementary material.

\section{Proofs for the Relative Gaussian Mechanism}
\label{app:proofs}

\subsection{Bounding the noise scale ratio}
We start by the following simple lemma, that will allow us to bound the domain of admissible $\alpha$.

\begin{lemma}\label{lemma:ratio}
        If $\sigma^2 \geq \frac{\gamma(1 + \simi^{-1})}{ 2\simi + \simi^2} \Rrel^2$, then $(1 + \eta)^{-2} \leq \frac{\gamma \sqnorm{\rx} + \sigma^2}{\gamma \sqnorm{\ry} + \sigma^2} \leq (1 + \simi)^2$.
\end{lemma}

\begin{proof}
    \begin{align*}
        \frac{\gamma \sqnorm{\rx} + \sigma^2}{\gamma \sqnorm{\ry} + \sigma^2} &= \frac{\gamma \sqnorm{\ry + \rx - \ry} + \sigma^2}{\gamma \sqnorm{\ry} + \sigma^2}\\
        &\leq \frac{\gamma (1 + \simi)\sqnorm{\ry} + \gamma (1 + \simi^{-1}) \sqnorm{\rx - \ry} + \sigma^2}{\gamma \sqnorm{\ry} + \sigma^2}\\
        &\leq \frac{\gamma (1 + \simi)\sqnorm{\ry} + \gamma (1 + \simi^{-1}) (\simi^2 \sqnorm{\ry} + \Rrel^2) + \sigma^2}{\gamma \sqnorm{\ry} + \sigma^2}.
    \end{align*}
    The result then follows from using the bound on $\sigma^2$ to factor $\gamma \sqnorm{\ry} + \sigma^2$ in the numerator. The other side (lower bound) is obtained by inverting $\rx$ and $\ry$.
\end{proof}

\subsection{Rényi divergence of two Gaussians}
\label{sec:renyi-divergence-two}

Recall that for $\alpha > 1$ and two distributions, $P$ and $Q$, the Rényi divergence is
\begin{align*}
  \cD_\alpha(P||Q)
  & = \frac{1}{\alpha - 1}
    \log \int \frac{P(x)^\alpha}{Q(x)^\alpha} dQ(x)
    = \frac{1}{\alpha - 1}
    \log \int \frac{P(x)^\alpha}{Q(x)^{\alpha-1}} dx
    \enspace.
\end{align*}

\begin{lemma} \label{lemma:two_gaussians}
Let $P$ and $Q$ be Gaussian distributions of dimension $d$ centered in $\mu_1$ and $\mu_2$ with variance $\sigma_1^2 \Id$ and $\sigma_2^2 \Id$. Then, assuming that $\alpha\sigma_2^2 + (1-\alpha)\sigma_1^2 > 0$, 
    \begin{equation}
    \label{eq:lemma2}
    \cD_\alpha(P||Q) = \frac{\alpha\norm{\mu_1-\mu_2}^2}{2(\alpha \sigma_2^2 + (1-\alpha)\sigma_1^2)} + \frac{d}{\alpha - 1}
    \log
    \left(\frac{\sigma_1^{1-\alpha}\sigma_2^\alpha}
    {\sqrt{\alpha \sigma_2^2 + (1 - \alpha) \sigma_1^2}}
    \right)
  \enspace.
    \end{equation}
\end{lemma}

Remark that when $\sigma_1=\sigma_2$, we recover the divergence of the standard Gaussian mechanism.

\begin{proof}
\begin{align*}
  \cD_\alpha(P||Q)
  & = \frac{1}{\alpha - 1}
    \log \sqrt{\frac{\sigma_2^{2d(\alpha-1)}}{(2\pi)^d \sigma_1^{2d\alpha}}}
    \int \exp\left(
    - \frac{\alpha \norm{u - \mu_1}^2}{2 \sigma_1^2}
    - \frac{(1-\alpha) \norm{u - \mu_2}^2}{2 \sigma_2^2}
    \right)
    du
    \enspace.
\end{align*}
We first compute the one-dimensional integral
\begin{align*}
  & \int_{-\infty}^{+\infty}
    \exp\left(
    -\frac{\alpha (u - \mu_1)^2}{2\sigma_1^2}
    -\frac{(1-\alpha) (u - \mu_2)^2}{2\sigma_2^2}
    \right) \\
  & = \int_{-\infty}^{+\infty}
    \exp\left(
    - \left(\frac{\alpha}{2\sigma_1^2} + \frac{1-\alpha}{2\sigma_2^2}\right) u^2
    + \left(\frac{\alpha\mu_1}{\sigma_1^2} + \frac{(1-\alpha)\mu_2}{\sigma_2^2}\right) u
    - \frac{\alpha\mu_1^2}{2\sigma_1^2} - \frac{(1-\alpha)\mu_2^2}{2\sigma_2^2}
    \right) \\
  & = \int_{-\infty}^{+\infty}
    \exp\left(
    - \left(\frac{\alpha \sigma_2^2 + (1-\alpha)\sigma_1^2}{2\sigma_1^2\sigma_2^2}\right) u^2
    + \left(\frac{\alpha\mu_1\sigma_2^2 + (1-\alpha)\mu_2\sigma_1^2}{\sigma_1^2\sigma_2^2}\right) u
    - \frac{\alpha\mu_1^2\sigma_2^2 + (1-\alpha)\mu_2^2\sigma_1^2}{2\sigma_1^2\sigma_2^2}
    \right)
    \enspace.
\end{align*}
Now, since
$\displaystyle \int_{-\infty}^{+\infty} \exp(-(au^2+bu+c))du =
\sqrt{\frac{\pi}{a}} \exp\Big(\frac{b^2}{4a} - c\Big)$, we have after
simplification, \textbf{and assuming
  $\alpha\sigma_2^2 + (1-\alpha)\sigma_1^2 > 0$},
\begin{align*}
  \int_{-\infty}^{+\infty}
    \exp\left(
    -\frac{\alpha (u - \mu_1)^2}{2\sigma_1^2}
    -\frac{(1-\alpha) (u - \mu_2)^2}{2\sigma_2^2}
    \right)
  & = \sqrt{
    \frac
    {2\pi\sigma_1^2\sigma_2^2}
    {\alpha \sigma_2^2 + (1 - \alpha) \sigma_1^2}
    }
    \exp\left(
    - \frac{\alpha(1-\alpha)(\mu_1-\mu_2)^2}{2(\alpha \sigma_2^2 + (1-\alpha)\sigma_1^2}
    \right) \\
  & =
    \frac
    {\sqrt{2\pi} \sigma_1\sigma_2}
    {\sqrt{\alpha \sigma_2^2 + (1 - \alpha) \sigma_1^2}}
    \exp\left(
    - \frac{\alpha(1-\alpha)(\mu_1-\mu_2)^2}{2(\alpha \sigma_2^2 + (1-\alpha)\sigma_1^2}
    \right)
  \enspace.
\end{align*}
Back to our divergence, we use the $1$-dimensional computations on each dimensions and obtain
\begin{align*}
  \cD_\alpha(P||Q)
  & = \frac{1}{\alpha - 1}
    \log \left(
    \sqrt{\frac{\sigma_2^{2d(\alpha-1)}}{(2\pi)^d \sigma_1^{2d\alpha}}}
    \prod_{j=1}^d \frac
    {\sqrt{2\pi} \sigma_1\sigma_2}
    {\sqrt{\alpha \sigma_2^2 + (1 - \alpha) \sigma_1^2}}
    \exp\left(
    - \frac{\alpha(1-\alpha)(\mu_{1,j}-\mu_{2,j})^2}{2(\alpha \sigma_2^2 + (1-\alpha)\sigma_1^2)}
    \right) \right) \\
  & = \frac{1}{\alpha - 1}
    \log \left(
    \prod_{j=1}^d \frac{\sigma_1^{1-\alpha}\sigma_2^\alpha}
    {\sqrt{\alpha \sigma_2^2 + (1 - \alpha) \sigma_1^2}}
    \exp\left(
    - \frac{\alpha(1-\alpha)(\mu_{1,j}-\mu_{2,j})^2}{2(\alpha \sigma_2^2 + (1-\alpha)\sigma_1^2)}
    \right) \right) \\
  & = \frac{1}{\alpha - 1}
    \log \left(
    \left(\frac{\sigma_1^{1-\alpha}\sigma_2^\alpha}
    {\sqrt{\alpha \sigma_2^2 + (1 - \alpha) \sigma_1^2}}
    \right)^d
    \exp\left(
    - \frac{\alpha(1-\alpha)\norm{\mu_1-\mu_2}^2}{2(\alpha \sigma_2^2 + (1-\alpha)\sigma_1^2)}
    \right) \right)\\
  & = \frac{d}{\alpha - 1}
    \log
    \left(\frac{\sigma_1^{1-\alpha}\sigma_2^\alpha}
    {\sqrt{\alpha \sigma_2^2 + (1 - \alpha) \sigma_1^2}}
    \right)
    + \frac{\alpha\norm{\mu_1-\mu_2}^2}{2(\alpha \sigma_2^2 + (1-\alpha)\sigma_1^2)}
  \enspace.
\end{align*}
\end{proof} 

\subsection{Privacy guarantees (Theorem~\ref{thm:renyi-sgm})}
We would now like to apply Lemma~\ref{lemma:two_gaussians} where $\mu_1 = \rx$, $\sigma_1^2 = \gamma \sqnorm{\rx} + \sigma^2$, $\mu_2 = \ry$ and $\sigma_2^2 = \gamma \sqnorm{\ry} + \sigma^2$. 

\textbf{Verifying the condition on $\alpha$.} To this end, we first need to verify that $\alpha \sigma_2^2 + (1 - \alpha) \sigma_1^2 > 0$, or equivalently that $\sigma_2^2 / \sigma_1^2 \geq 1 - \alpha^{-1}$. If applicable, Lemma~\ref{lemma:ratio} would directly give us that this is true as long as $(1 + \simi)^{-2} \geq 1 - \alpha^{-1}$, which leads to the bound: 
\begin{equation}
    \alpha^{-1} \geq 1 - (1 + \simi)^{-2} = \frac{\simi(2 + \simi)}{(1 + \simi)^2},
\end{equation}
so in the end: 
\begin{equation}
    \alpha \leq \frac{(1 + \simi)^2}{\simi(2 + \simi)}.
\end{equation}
This is equivalent to $\alpha - 1 \leq \frac{1}{\simi(2 + \simi)}$, which is the condition from Theorem~\ref{thm:renyi-sgm}. In order to apply Lemma~\ref{lemma:ratio}, we need to verify that 
\begin{equation}
    \sigma^2 \geq \gamma \frac{1 + \simi^{-1}}{2\simi + \simi^2} \Rrel^2 = \frac{\gamma \Rrel^2}{\simi^2}\frac{1 + \simi}{2 + \simi}.
\end{equation}
This is automatically verified for the choice of $\sigma$ from Theorem~\ref{thm:renyi-sgm} since 
\begin{equation}
    1 - \simi(\alpha - 1) \geq 1 - (2 + \simi)^{-1} = \frac{1 + \simi}{2 + \simi}.
\end{equation}

\textbf{Bounding the main term in \eqref{eq:lemma2}.}
Now that we verified that we can apply Lemma~\ref{lemma:two_gaussians}, we can use it to bound the divergence.
To bound the first term in \eqref{eq:lemma2}, we start by recalling that, by Young's inequality, 
\begin{align}
  \norm{\rx}_2^2
  & \le (1+\simi) \norm{\ry}_2^2
    + (1+\simi^{-1}) \norm{\rx - \ry}_2^2
    \enspace.
\end{align}
Using this inequality, we can lower bound the denominator
\begin{align}
  &\alpha \sigma_2^2 + (1-\alpha)\sigma_1^2\\
  & = \sigma^2
    + \gamma \alpha \norm{\rx}_2^2
    - \gamma (\alpha - 1) \norm{\ry}_2^2 \\
  & \ge \sigma^2
    + \gamma \alpha \norm{\ry}_2^2
    - \gamma (\alpha - 1)(1 + \simi) \norm{\ry}_2^2
    - \gamma (\alpha - 1)(1 + \simi^{-1}) \norm{\rx - \ry}_2^2 \\
  & = \sigma^2
    + (\gamma \alpha - \gamma (\alpha - 1)(1 + \simi) \norm{\ry}_2^2
    - \gamma (\alpha - 1)(1 + \simi^{-1}) \norm{\rx - \ry}_2^2 \\
  & = \sigma^2
    + \gamma (1 - (\alpha - 1) \simi) \norm{\ry}_2^2
    - \gamma (\alpha - 1)(1 + \simi^{-1}) \norm{\rx - \ry}_2^2
\end{align}
Now, remark that relative sensitivity gives
$ \norm{\ry}_2^2 \ge \frac{1}{\simi^2} \left( \norm{\rx - \ry}_2^2 -
  R^2 \right)$, which in turn yields
\begin{align}
  &\alpha \sigma_2^2 + (1-\alpha)\sigma_1^2
    \\
    &\ge \sigma^2
    + \frac{\gamma}{\simi^2} (1 - (\alpha - 1) \simi) \left(
    \norm{\rx - \ry}_2^2 - R^2
    \right)
    - \gamma (\alpha - 1)(1 + \simi^{-1}) \norm{\rx - \ry}_2^2 \\
  & = \sigma^2 -  \frac{\gamma}{\simi^2} (1 - (\alpha - 1) \simi) R^2
    + \left(
    \frac{\gamma}{\simi^2} (1 - (\alpha - 1) \simi)
    - \gamma (\alpha - 1)(1 + \simi^{-1})
    \right)
    \norm{\rx - \ry}_2^2
    \enspace.
\end{align}
We now use the condition on $\alpha$, which is that $\alpha - 1 = \frac{1 - \rho}{\eta(2 + \eta)}$ for some $\rho < 1$. 
If we further assume that
$\sigma^2 \ge \frac{\gamma}{\simi^2} (1 - (\alpha - 1) \simi) R^2$,
which is notably the case when
$\sigma^2 \ge \frac{\gamma R^2}{\simi^2}$, we obtain
\begin{align}
  \alpha \sigma_2^2 + (1-\alpha)\sigma_1^2
  & \ge \left(
    \frac{\gamma}{\simi^2} (1 - (\alpha - 1) \simi)
    - \gamma (\alpha - 1)(1 + \simi^{-1})
    \right)
    \norm{\rx - \ry}_2^2 \\
  & = \left(
    \frac{\gamma}{\simi^2}
    - \frac{\gamma}{\simi} (\alpha - 1)\left(2 + \simi\right)
    \right)
    \norm{\rx - \ry}_2^2 \\
  & = \frac{\gamma \rho}{\simi^2} \norm{\rx - \ry}_2^2 \enspace.
\end{align} 
Putting everything back together, we get that for $ \simi(2 + \simi)(\alpha - 1) \leq 1$, and $\sigma^2 \geq \gamma (1 - (\alpha - 1) \simi)R^2 / \simi^2$:
\begin{align}
  \frac
  {\alpha \norm{\rx - \ry}_2^2}
  {2( \alpha \sigma_2^2 + (1-\alpha)\sigma_1^2)}
  & \le \frac{\alpha \simi^2}{2\gamma \rho} = \frac{\alpha \simi^2}{2\gamma} \frac{1}{1 - \simi(2 + \simi)(\alpha - 1)}
    \enspace.
\end{align}

\textbf{Bounding the log term in \eqref{eq:lemma2}.}
Remark that, as a consequence of Lemma~\ref{lemma:ratio}, we have that $\sigma_2^2 = r\sigma_1^2$, for some
$r \le (1 + \simi)^2$. In particular, we have
\begin{align}
  \log\left(
  \frac{\sigma_1^{1-\alpha}\sigma_2^\alpha}
  {\sqrt{\alpha \sigma_2^2 + (1 - \alpha) \sigma_1^2}}
  \right)
  & =
    \log\left(
    \frac{r^{\alpha/2}\sigma_1}
    {\sqrt{1 - \alpha + \alpha r} \sigma_1}
    \right)
    = \frac{\alpha-1}{2} \log(r)
    - \frac{1}{2} \log\left(\frac{1+(r-1)\alpha}{r}\right)
    \enspace.
\end{align}
The two terms can be upper bounded using $\log(1+x) \le x$ for $x \ge -1$,
\begin{align}
  \frac{\alpha - 1}{2} \log(r)
  & \le \frac{\alpha - 1}{2} (r - 1)
    \enspace,
\end{align}
and $\log(1+x) \ge \frac{x}{1+x}$ for $x \ge -1$,
\begin{align}
  - \frac{1}{2} \log\left(\frac{1+(r-1)\alpha}{r}\right)
  & = - \frac{1}{2} \log\left(1 + \frac{(\alpha - 1)(r-1)}{r}\right) \\
  & \le - \frac{1}{2} \frac{(\alpha - 1)(r-1)}{r + (\alpha - 1)(r-1)} \\
  & = - \frac{1}{2} \frac{(\alpha - 1)(r-1)}{1 + \alpha(r-1)}
    \enspace.
\end{align}
Overall, we obtain
\begin{align}
  \log\left(
  \frac{\sigma_1^{1-\alpha}\sigma_2^\alpha}
  {\sqrt{\alpha \sigma_2^2 + (1 - \alpha) \sigma_1^2}}
  \right)
  & \le \frac{(\alpha - 1)(r-1)}{2} \left(1 - \frac{1}{1 + \alpha(r-1)}\right)
    = \frac{\alpha(\alpha - 1)(r-1)^2}{2(1 + \alpha(r-1))}
    \enspace,
\end{align}
Since $(1 + \simi)^{-2} \le r \le (1 + \simi)^2$, we have that 
\begin{equation}
    |r - 1| \leq \max\left((1 + \simi)^2 - 1, 1 - (1 + \simi)^{-2}\right) = \max\left(\simi(2 + \simi), \frac{\simi(2 + \simi)}{(1 + \simi)^2}\right) \leq \simi(2 + \simi),
\end{equation}
so that in the end:
\begin{equation}
  \frac{d}{\alpha - 1}\log\left(
  \frac{\sigma_1^{1-\alpha}\sigma_2^\alpha}
  {\sqrt{\alpha \sigma_2^2 + (1 - \alpha) \sigma_1^2}}
  \right) \le \frac{\alpha d \simi^2 (2 + \simi)^2}{2\left(1 - \alpha\frac{\simi(2 + \simi)}{(1 + \simi)^2}\right)} =  \frac{\alpha d \simi^2 (2 + \simi)^2(1 + \simi)^2}{2(1 - (\alpha - 1)\simi(2 + \simi))}
    \enspace.
\end{equation}

\subsection{Tightness}
\label{app:tightness}

As explained in the main text, the link with local sensitivity implies the tightness of the first term in Theorem~\ref{thm:renyi-sgm} (the one which depends on $\gamma$). We now discuss the tightness of the second term, which comes from the $\log$ term in \eqref{eq:lemma2}. In particular, we write (similarly to the previous section): 
\begin{align}
  \log\left(
  \frac{\sigma_1^{1-\alpha}\sigma_2^\alpha}
  {\sqrt{\alpha \sigma_2^2 + (1 - \alpha) \sigma_1^2}}
  \right)
  & = \frac{\alpha}{2} \log(1 + u)
    - \frac{1}{2} \log\left(1+\alpha u\right) = g(u)
    \enspace,
\end{align}
with $u = r - 1$. We would like to find an expression for $g(u)$ when $u\approx 0$, which means that $r \approx 1$ and so $\simi$ is small. Differentiating with respect to $u$ leads to: 
\begin{equation}
    g^\prime(u) = \frac{\alpha}{2} \left(\frac{1}{1 + u} - \frac{1}{1 + \alpha u}\right).
\end{equation}
Note that $g(0) = 0$, and $g^\prime(0) = 0$, so we have to differentiate once again, leading to:
\begin{equation}
    g^{\prime\prime}(u) = \frac{\alpha}{2} \left(\frac{\alpha}{(1 + \alpha u)^2} - \frac{1}{(1 + u)^2} \right).
\end{equation}
In particular, $g^{\prime\prime}(0) = \frac{\alpha (\alpha - 1)}{2}$, so when $r \approx 1$, 
\begin{equation}
    g(r) = \frac{\alpha(\alpha - 1)}{4}(r - 1)^2 + O((r-1)^3).
\end{equation}
The leading term is the same expression as we had before, up to a factor $1/2$. In particular, the bounding from the previous subsection is tight up to a factor $1/2$.

Ideally, we would like to use this approximation as $g(r)$. In order to do this, we have to show that $g^{\prime\prime\prime}(r - 1) \leq 0$ for  all $(\alpha, \simi)$ we consider. Let us differentiate one last time, leading to: 
\begin{equation}
    g^{\prime\prime\prime}(u) = \alpha \left(\frac{1}{(1 + u)^3} - \frac{\alpha^2}{(1 + \alpha u)^3} \right).
\end{equation}
One can remark that $g^{\prime\prime\prime}(u) \leq 0$ for $u \leq 0$. However, it can be that $g^{\prime\prime\prime}(u) > 0$ for some $u > 0$. More specifically, if $\alpha$ is large enough ($\alpha \geq 2$ for instance), then one can show that $g^{\prime\prime\prime}(u) \leq 0$ for the range of $u$ that we consider (i.e., $u = r - 1 \leq (1+ \simi)^2 - 1 = \simi(2 + \simi) \leq (\alpha - 1)^{-1}$. Yet, this does not hold for all $\alpha$, and $g^{\prime\prime\prime}((\alpha - 1)^{-1}) > 0$ for $\alpha = 3/2$ for instance. This is why we keep the result which is off by a factor up to $2$ for large $\alpha$, but works for any value of $(\alpha, \simi)$ in our range.

\subsection{Comparison with differential privacy}
\label{sup:proof-of-conv-to-dp}
\begin{corollary*}
Assuming that $\gamma \le \tfrac{1}{4 (2+\eta)^2 \log(1/\delta)}$ or $d \ge 4 \log(1/\delta) / (1 + \simi)^2$, and that $0 <\delta \le 1$, the Relative Gaussian Mechanism is $(\epsilon,\delta)$-DP with
    \begin{align}
        \epsilon = 
        \convtodp + 2 \sqrt{ \convtodp \log(1/\delta) }
        , \quad \text{ where }
        \convtodp = \tfrac{\eta^2}{\gamma} + \eta^2 d (2+\eta)^2(1 + \eta)^2.
    \end{align}
\end{corollary*}

\begin{proof}
Applying the standard conversion result from RDP to $(\epsilon,\delta)$-DP, we get that the Relative Gaussian Mechanism is $(\epsilon, \delta)$-DP for
\begin{align}
    \epsilon
    & =
    \min_{1 \le \alpha \le \frac{(1+\eta)^2}{2\eta + \eta^2}}
    \left\{
    g(\alpha) :=
    \frac{\alpha \eta^2}{2 \gamma} \frac{1 + \gamma d (2 + \eta)^2(1 + \eta)^2}{1 - \eta(2+\eta)(\alpha - 1)} 
    + \frac{\log(1/\delta)}{\alpha-1}
    \right\}.
\end{align}
Restricting to $\alpha \le 1+\frac{1}{2\simi(2+\simi)} \le \frac{(1+\simi)^2}{\simi(2+\simi)}$, we have $\eta(2+\eta)(\alpha-1) \le 1/2$. We can therefore upper bound $g(\alpha)$ by
\begin{align}
    g(\alpha) 
    & \le
    \frac{\alpha\eta^2}{\gamma}
    + \alpha\eta^2 d (2+\eta)^2(1 + \simi)^2
    + \frac{\log(1/\delta)}{\alpha-1}
    =
    \alpha \convtodp
    + \frac{\log(1/\delta)}{\alpha-1}
    .
\end{align}
This upper bound is minimal when $\convtodp = \frac{\log(1/\delta)}{(\alpha - 1)^2}$, that is $\alpha = 1 + \sqrt{ \frac{\log(1/\delta)}{\convtodp} }$. Using this value of $\alpha$, we have
\begin{align}
    \epsilon 
    & \le g(\alpha)
    \le \convtodp + \sqrt{\convtodp \log(1/\delta)} + \sqrt{\convtodp \log(1/\delta)}
    \le \convtodp + 2 \sqrt{\convtodp} \log(1/\delta)
    .
\end{align}
Note that we could choose this value of $\alpha$ since either $\gamma \le \frac{1}{4(2+\eta)^2\log(1/\delta)}$ or $d \ge 4 \log(1/\delta) / (1 + \simi)^2$. These inequalities indeed implies that $\alpha = 1 + \sqrt{\frac{\log(1/\delta)}{\convtodp}} \le 1 + \frac{1}{2\simi(2 + \simi)}$, which is equivalent to
\begin{align}
    \frac{\convtodp}{4 \eta^2 (2 + \eta)^2} 
    =
    \frac{1}{4 \gamma (2 + \eta)^2} 
    + \frac{(1 + \simi)^2 }{4} d
    \ge 
    \log(1/\delta)
    .
\end{align}
\end{proof}

\subsection{Links with the Gaussian Smooth Sensitivity}
\label{app:gss}
Gaussian Smooth Sensitivity~\citep[Section 5]{bun2018composable} is an analog of the smooth sensitivity framework~\citep{smooth_sens}. The main result writes as follows (changed to be compatible with our notations): 

\begin{proposition}[\citet{bun2018composable}]
    \label{prop:bun}For all neighboring $x \sim y$, if $|\cR(x) - \cR(y)| \leq \Delta_f e^{g(x)/2}$ and $|g(x) - g(x^\prime)| \leq \Delta_g$, then if  the Gaussian Smooth Sensitivity, $GSS$, writes $GSS(\cR)(x) = \cR(x) + \cN(0, e^{g(x)})$, then $GSS$ satisfies $(\Delta_f^2 + \Delta_g^2, 1/(2\Delta_g))$-tCDP.  
\end{proposition} 

Let us consider $g(x) = \log\left(\gamma \sqnorm{\cR(x)} + \sigma^2\right)$. Then, $GSS$ is equivalent to $\RGM$, and we have that 
\begin{equation}
    |g(x) - g(y)| = \log\left(\frac{\gamma \sqnorm{\cR(x)} + \sigma^2}{\gamma \sqnorm{\cR(x^\prime)} + \sigma^2}\right)
\end{equation}
In particular, since $\log$ is an increasing function, we can use Lemma~\ref{lemma:ratio} and write:
\begin{equation}
    |g(x) - g(y)| \leq \log((1 + \eta)^2) \leq 2 \eta,
\end{equation}
so that $\Delta_g = 2\eta$ is a valid bound. Then, if we choose $\sigma^2 = \gamma\eta^{-2} \Rrel^2$ (which satisfies the condition in Theorem~\ref{thm:renyi-sgm}), relative sensitivity gives us: 
\begin{equation}
    |\cR(x) - \cR(y)|^2 \leq \eta^2 \sqnorm{\cR(x)} + \Rrel^2 = \frac{\eta^2}{\gamma} \left(\gamma\sqnorm{\cR(x)} + \sigma^2 \right) = \frac{\eta^2}{\gamma} e^{g(x)}.
\end{equation}
In particular, this means that we can take $\Delta_f^2 = \frac{\eta^2}{\gamma}$.

Thus, Proposition~\ref{prop:bun} ensures that $\RGM$ guarantees $(4 \eta^2 + \frac{\eta^2}{\gamma}, 1./ (4\eta))$-tCDP. 

Alternatively, we can directly transform Theorem~\ref{thm:renyi-sgm} into a $(\frac{\eta^2}{\gamma} + \eta^2 d(2+\eta)^2 / 2], 1 + (2\eta(2+\eta))^{-1})$-tCDP result by eliminating $\alpha$ from the $\varepsilon$ bound, using that $\alpha - 1 < (2\eta + \eta^2)^{-1}$, and so $\left[1 - \eta(\alpha-1)(2 + \eta)\right]^{-1} \leq 2$.

In particular, we see that the guarantees we obtain are very similar for $\eta < 1$, which is the setting for which our results have been optimized (we have used inequalities that are tight for $\eta$ small, but could also improve them for larger $\eta$). Indeed, our bound improves the guarantees obtained by $GSS$ by a factor of $2$ for small $\eta$. 

We have seen that if $\cR(x) \in \R$ (or equivalently $d=1$), the results of Theorem~\ref{thm:renyi-sgm} can be recovered using Proposition~\ref{prop:bun}. However, Theorem~\ref{thm:renyi-sgm} extends these results for $\cR(x) \in \R^d$, which is essential to use relative sensitivity for gradient descent.

\section{Relative L2 sensitivity for releasing gradients}
\label{app:gradients_quadratics}

Before we start this section, we introduce a few notions, that will be useful in particular in subsections~\ref{app:utility} and~\ref{app:general_functions}. The first is the notion of Bregman Divergence, which is defined for a function $h$ and for points $\theta, \theta^\prime$ as:
\begin{equation}
    D_h(\theta, \theta^\prime) = h(\theta) - h(\theta^\prime) - \nabla h(\theta^\prime)^\top (\theta - \theta^\prime). 
\end{equation}
One can see that $h$ is $L$-smooth and $\mu$-strongly convex is equivalent to having for all $\theta, \theta^\prime \in \R^d$:
\begin{equation}
    \frac{\mu}{2} \sqnorm{\theta - \theta^\prime} \leq D_h(\theta, \theta^\prime) \leq  \frac{L}{2} \sqnorm{\theta - \theta^\prime}.
\end{equation}
The Bregman divergence also has nice properties w.r.t. convex conjugation. More specifically, the convex conjugate $h^*$ of $h$ is defined as $h^*(\theta) = \arg\max_u \theta^\top u - f(u)$. Then, it holds that:
\begin{equation}
    D_h(\theta, \theta^\prime) = D_{h^*}(\nabla h(\theta^\prime), \nabla h(\theta)).
\end{equation}
Bregman divergences are often linked in convex optimization to the notion of \emph{relative} smoothness, which generalizes regular smoothness. In particular, a function $f$ is said to be $\Lrel$-smooth with respect to $h$ if for all $\theta, \theta^\prime \in \R^d$,
\begin{equation}
    D_f(\theta, \theta^\prime) \leq \Lrel D_h(\theta, \theta^\prime).
\end{equation}
This recover standard smoothness when $h = \frac{1}{2}\sqnorm{\cdot}$, and will play a key role in showing our relative sensitivity assumption for gradient descent with general functions. 

\subsection{Utility (Proof of Theorem~\ref{thm:utility})} 
\label{app:utility}

We provide below the proof of Theorem~\ref{thm:utility}. We also prove that under the same assumptions (but this result can also be used for $\mu = 0$, \emph{i.e.} $f(\cdot; D)$ is convex), we have that:
\begin{equation}\label{eq:utility_convex}
     f(\bar{\theta}_t) - f(\theta_\star) \leq \frac{\sqnorm{\theta_0 - \theta_\star}}{2\tau t} + \frac{\tau \sigma^2}{2},
\end{equation}
where $\bar{\theta}_t = \frac{1}{t} \sum_{k=0}^{t-1} \theta_k$.

\begin{proof}
We write $f(\theta) = f(\cdot; D)$ for simplicity. In this case, for all $t \geq 0$,
    \begin{align*}
        \esp{\sqnorm{\theta_{t+1} - \theta_\star}} &\leq \sqnorm{\theta_t - \theta_\star} - 2 \tau \esp{g_t^\top(\theta_t - \theta_\star)} + \tau^2 \esp{\sqnorm{g_t}}\\
        &= \sqnorm{\theta_t - \theta_\star} - 2 \tau \nabla f(\theta_t)^\top(\theta_t - \theta_\star) + \tau^2 \left[\sqnorm{\nabla f(\theta_t)} + \esp{\sqnorm{\xi_t}}\right]\\
        &= \sqnorm{\theta_t - \theta_\star} - 2 \tau \nabla f(\theta_t)^\top(\theta_t - \theta_\star) + \tau^2 (1 + \gamma)\sqnorm{\nabla f(\theta_t)} + \tau^2 \sigma^2.
    \end{align*}
    We now use the smoothness of $f$, which ensures that $\sqnorm{\nabla f(\theta_t)} \leq 2L D_f(\theta_\star, \theta_t)$, and obtain: 
    \begin{equation}\label{eq:main_utility}
    \esp{\sqnorm{\theta_{t+1} - \theta_\star}} \leq \sqnorm{\theta_t - \theta_\star} - 2\tau D_f(\theta_t, \theta_\star)- 2 \tau(1 - (1 + \gamma)\tau L) D_f(\theta_\star, \theta_t) + \tau^2 \sigma^2.
    \end{equation}
    We can then use the $\mu$-strong convexity of $f$, which yields $2D_f(\theta_t, \theta_\star) \geq \mu\sqnorm{\theta_t - \theta_\star}$, and so:
    \begin{equation}\label{eq:main_str_convex}
        \esp{\sqnorm{\theta_{t+1} - \theta_\star}} \leq (1 - \tau \mu)\sqnorm{\theta_t - \theta_\star} - 2 \tau(1 - (1 + \gamma)\tau L) \left[ f(\theta_t) - f(\theta_\star)\right] + \tau^2 \sigma^2.
    \end{equation}
    By taking $\tau \leq [(1 + \gamma)L]^{-1}$, using that $D_f \geq 0$ since $f$ is convex, and chaining the inequalities, we obtain: 
    \begin{equation}
        \esp{\sqnorm{\theta_t - \theta_\star}} \leq (1 - \tau \mu)^t \sqnorm{\theta_0 - \theta_\star} + \frac{\tau \sigma^2}{\mu}.
    \end{equation}
    In the convex case ($\mu = 0$), we go back from~\eqref{eq:main_utility}, use the same step-size condition, and rewrite it as: 
    \begin{equation}
        f(\theta_t) - f(\theta_\star) \leq \frac{\esp{ \sqnorm{\theta_t - \theta_\star} - \sqnorm{\theta_{t+1} - \theta_\star}}}{2\tau} + \frac{\tau \sigma^2}{2}.
    \end{equation}
    We now write (telescoping sum):
    \begin{equation}
        \frac{1}{t} \sum_{k=0}^{t-1} f(\theta_k) - f(\theta_\star) \leq \frac{\sqnorm{\theta_0 - \theta_\star}}{2\tau t} + \frac{\tau \sigma^2}{2}.
    \end{equation}
    Equation~\eqref{eq:utility_convex} is obtained by convexity of $f$.
\end{proof}

\subsection{Bounding relative sensitivity for general functions}
\label{app:general_functions}

Let $f,f^\prime$ be two functions on neighboring datasets (as defined in the main text), such that $f - f^\prime = f_0 - f_0^\prime$. Let $f$ and $f^\prime$ be $\mu$-strongly-convex, $f_0$ and $f_0^\prime$ be $L$-smooth, and $f$ and $f_0^\prime$ be $\Lrel$-relatively smooth \emph{w.r.t.} \textbf{both} $f$ and $f^\prime$. In this case, we have that:
\begin{align*}
    \|\nabla f&(\theta) - \nabla f^\prime(\theta)\|^2 = \frac{1}{n^2}\sqnorm{\nabla f_0(\theta) - \nabla f_0^\prime(\theta)}\\
    &= \frac{1}{n^2}\sqnorm{\nabla f_0(\theta) - \nabla f_0^\prime(\theta_\star) - [\nabla f_0^\prime(\theta) - \nabla f_0^\prime(\theta_\star)] + \nabla f_0(\theta_\star) - \nabla f_0^\prime(\theta_\star)}\\
    &= \frac{3}{n^2}\sqnorm{\nabla f_0(\theta) - \nabla f_0^\prime(\theta_\star)} + \frac{3}{n^2}\sqnorm{\nabla f_0^\prime(\theta) - \nabla f_0^\prime(\theta_\star)} + \frac{3}{n^2}\sqnorm{\nabla f_0(\theta_\star) - \nabla f_0^\prime(\theta_\star)}.
\end{align*}
Then, using successively the $L$-smoothness of $f_0$, the $\Lrel$-relative smoothness of $f_0$ \emph{w.r.t.} $f$, and strong convexity of $f$, we get:
\begin{equation}
    \sqnorm{\nabla f_0(\theta) - \nabla f_0^\prime(\theta_\star)} \leq 2L D_{f_0}(\theta, \theta_\star) \leq 2L \Lrel D_f(\theta, \theta_\star) \leq \frac{L\Lrel}{\mu} \sqnorm{\nabla f(\theta)}. 
\end{equation}
We can then use the same bound for the $f_0^\prime$ term, and for making $f^\prime$ appear. Compared to direct bounding without relative smoothness, we have replaced a condition number $L/\mu$ by a relative smoothness term $\Lrel$, which is generally much smaller. We then obtain that $\simi^2 = O(\kappa \Lrel / n^2)$, much like in Equation~\eqref{eq:condition_Lrel}.

\subsection{Proof for orthogonal data}
\label{app:orthogonal_data}

Let us assume that the data is orthogonal, \emph{i.e.} that either $X_i^\top X_j = \sqnorm{X_i}$ or $X_i^\top X_j = 0$. This is actually slightly stronger, because we also assume that there is no norm spread for a given direction. This could be relaxed so that results would depend on the average norm, but we keep it simple here.

Consider that at least half of the dataset is fixed, and contains all different $X_i$ in equal proportions. This means that the weight of each $X_i$ is $d^{-1}$, since there are exactly $d$ different ones. Indeed, there cannot be more than $d$ (otherwise the orthogonality constraint would be violated), and if there are less than $d$ we can just restrict to the relevant subspace. In this case, using that half of the data (\emph{w.l.o.g.} is fixed, we have that: 
\begin{align*}
A &= \frac{1}{n} \sum_{i=1}^n X_i X_i^\top \succcurlyeq \frac{1}{n} \sum_{i=1}^{n/2} X_i X_i^\top \succcurlyeq \frac{1}{2d} \sum_{i=1}^d X_i X_i^\top, 
\end{align*}
where the last line comes from the fact that each $X_i$ is represented $n/d$ times, and we implicitly assume (again, \emph{w.l.o.g.}) that the $i$ first samples are all orthogonal to one another. In particular, for $j \in \{1, \dots, d\}$, 
\begin{equation*}
    \sqnorm{X_i}X_i^\top A^{-2} X_i \leq 2d \sqnorm{X_i} X_i^\top \left(\sum_{j=1}^d X_j X_j^\top\right)^{-2} \! \! \! X_i = 2d \sqnorm{X_i} X_i^\top \left(\frac{X_i X_i^\top}{\norm{X_i}^4}\right)^2 X_i = 2d.
\end{equation*}
Note in particular that the relative sensitivity is independent of the scale of each $X_i$. 

\subsection{Proof of Proposition~\ref{prop:ptr}}
\label{app:proof_ptr}

Following~\citet[Section 3.2]{vadhan2017complexity}, the Propose-Test-Release mechanism is based on evaluating $\Delta$, the minimum number of points that we need to change from $X$ to obtain $\tilde{A}^\prime$ such that $\tilde{A}^\prime - \rho C \preccurlyeq 0$. Then, we privately test that $\Delta > 1$ by adding well-calibrated Laplacian noise to it. If the test fails, then this means that the threshold $\rho$ that we proposed is too low. If it passes, it means that $\rho$ can be used as the sensitivity since the neighbours of $\rho$ also satisfy this threshold. 

However, computing $\Delta$ might require enumerating all the possibilities to remove points from $\tilde{A}$. We now show that this can be done efficiently in our case. 

Here, we only consider different datasets that have \textbf{the same number of points $n$.} This technicality can be lifted, but at the expense of complicating the proof by putting a restriction that depends on $n$. Basically, another way of making $\tilde{A} = \frac{1}{n}XX^\top$ small would be to just add data points that hardly contribute to the covariance while increasing $n$, and we would have to take this case into account.

Now, let $I_R$ be the indices of points that we remove from $X$, and let $I_R^c$ be the points that we add instead, then we have that
\begin{equation}
    \tilde{A}^\prime = \tilde{A} - \frac{1}{n}\sum_{i \in I_R} X_i X_i^\top + \frac{1}{n}\sum_{i \in I_R^c} X_i X_i^\top \succcurlyeq \tilde{A} - \frac{1}{n}\sum_{i \in I_R} X_i X_i^\top  
\end{equation}
In particular, $\{|I_R|, \tilde{A}^\prime - \rho C \preccurlyeq 0\} \subset \{|I_R|, \tilde{A} - \frac{1}{n}\sum_{i \in I_R} X_i X_i^\top  - \rho C \preccurlyeq 0\}$, and so:
\begin{equation}
    \Delta = \min \{|I_R|, \tilde{A}^\prime - \rho C \preccurlyeq 0\} \geq \min \{|I_R|, \tilde{A} - \frac{1}{n}\sum_{i \in I_R} X_i X_i^\top  - \rho C \preccurlyeq 0\}
\end{equation}
If $\tilde{A} - \rho C$ is not positive semi-definite then $\Delta = 0$. Otherwise, it has a unique PSD square root $Q$, so let us define $Y$ such that:
\begin{equation}
    \tilde{A} - \rho C = Q^2, \qquad Y_i = Q^{-1} X_i.
\end{equation}
Then, we have that:
\begin{equation}
    E_+ = \left\{ I_R, \tilde{A} - \rho C - \frac{1}{n}\sum_{i \in I_R} X_i X_i^\top \succcurlyeq 0\right\} = \left\{I_R, n\Id - \sum_{i \in I_R} Y_i Y_i^\top \succcurlyeq 0\right\} = \left\{I_R, \lambda_{\max} \left(\sum_{i \in I_R} Y_i Y_i^\top\right) \leq n\right\}
\end{equation}
Let $E_- = \left\{I_R, \lambda_{\max} \left(\sum_{i \in I_R} Y_i Y_i^\top\right) > n\right\}$ be the complement of $E_+$. Note that at this stage we still have a combinatory problem: there are $\sum_{\ell=1}^k {n \choose \ell}$ possibilities to test to be sure that $\min \{ |I|, I \in E_-\} \geq k$. However, we can introduce a simple inequality at this point, and write that 
\begin{equation}
\lambda_{\max}\left(\sum_{i\in I_R} Y_i Y_i^\top\right) \leq \sum_{i\in I_R} \lambda_{\max} (Y_i Y_i^\top) = \sum_{i\in I_R} \|Y_i\|^2.    
\end{equation}
In particular, this means that
\begin{equation}
    E_- \subset \{I_R, \sum_{i\in I_R} \|Y_i\|^2 > n\},\text{ and so  }\Delta \geq \min  \{|I_R|, \sum_{i\in I_R} \|Y_i\|^2 > n\}
\end{equation}
It remains to observe that $\|Y_i\|^2 = X_i^\top Q^{-2} X_i = X_i^\top (\tilde{A} - \rho C)^{-1} X_i$, so that we never actually have to compute the matrix square root.

\subsection{Proof of Proposition~\ref{prop:ellipse_gaussian}}
\label{app:proof_clip_gaussian}
\begin{proof}
    The proof follows the following plan: 
    \begin{enumerate}
        \item The clipped points concentrate, \emph{i.e.},
        \begin{equation}
        p(\norm{\tilde{A} - \esp{\tilde{A}}} \geq \reg) \leq \dproba \text{ for } n \geq 4 \log(2d/\dproba) / 9,
    \end{equation}
    with $\reg = 4 L R_c^2\sqrt{\frac{\log(2d / \dproba)}{n}}$. This in particular means that with probability at least $1-\dproba$, we have $\tilde{A} \succcurlyeq \esp{\tilde{A}} - \reg \Id$.
        \item $\esp{\tilde{A} - \reg \Id} = \bar{\rho} \Sigma$, so that with probability at least $1 - \dproba$, $\tilde{A} \succcurlyeq \bar{\rho} C$.
        \item PTR always succeeds with our choice of $\rho$, $\varepsilon$, $\delta$.
    \end{enumerate}

    \textbf{1 - Concentration of the covariance.}  For all samples in the clipped dataset, $\tilde{X}_i^\top \Sigma^{-1} \tilde{X}_i \leq R_c^2$ (by definition of clipping), and so $\lambda_{\max}(\Xtilde_i \Xtilde_i^\top) \leq R_c^2 L$. Let us define $S_i = \Xtilde_i \Xtilde_i^\top / n$. Then, 
    \begin{equation}
        \norm{S_k - \esp{S_k}} \leq 2 R_c^2 L / n, \text{ and } \nu = \sum_{i \in I_c} \esp{\sqnorm{S_i - \esp{S_i}}} \leq \frac{4L^2 R_c^4}{n}.
    \end{equation}
    In particular, we can then use~\citet[Corollary 6.1.2]{tropp2015introduction} to bound the tails of $\tilde{A} = \sum_{i = 1}^n S_i$ as: 
    \begin{equation}
        p(\norm{\tilde{A} - \esp{\tilde{A}}}) \geq t) \leq 2d \exp\left(-\frac{nt^2}{8L^2R_c^4 + 4 L R_c^2 t / 3}\right).
    \end{equation}
    In particular, if $t \leq 6 L R_c^2$ then 
    \begin{equation}
        p(\norm{\tilde{A} - \esp{\tilde{A}}} \geq t) \leq 2d \exp\left(-\frac{nt^2}{16L^2R_c^4}\right). 
    \end{equation}
    In this case, we have that 
    \begin{equation}
        p\left(\norm{\tilde{A} - \esp{\tilde{A}}} \geq 4 L R_c^2\sqrt{\frac{\log(2d / \dproba)}{n}}\right) \leq \dproba \text{ for } n \geq 4 \log(2d/\dproba) / 9. 
    \end{equation} 
    
    \textbf{2 - $\esp{\tilde{A}  - \reg \Id}$ is proportional to the covariance.} We write, using the change of variable $u = \Sigma^{-\frac{1}{2}} x$: 
    \begin{align*}
        \esp{\tilde{A} - \reg \Id} &= \frac{1}{(2\pi)^{d/2} |\Sigma|^{\frac{1}{2}}}\int_{\R^d} \left[\mathds{1}\{x^\top \Sigma^{-1} x \leq R_c^2\} + \mathds{1}\{x^\top \Sigma^{-1} x \geq R_c^2\} \frac{R_c^2}{x^\top \Sigma^{-1} x}\right] xx^\top e^{-\frac{x^\top \Sigma^{-1} x}{2}}{\rm d} x \\
        &= \Sigma^{\frac{1}{2}} \frac{1}{(2\pi)^{d/2}}\int_{\R^d} \left[\mathds{1}\{ \sqnorm{u} \leq R_c^2\} + \mathds{1}\{ \sqnorm{u} \geq R_c^2\} \frac{R_c^2}{\sqnorm{u}}\right] uu^\top e^{-\frac{\sqnorm{u}}{2}}{\rm d} u \Sigma^{\frac{1}{2}}\\
        &= \Sigma \times \frac{1}{(2\pi)^{d/2}}\int_{\R^d} \left[\mathds{1}\{ \sqnorm{u} \leq R_c^2\} + \mathds{1}\{ \sqnorm{u} \geq R_c^2\} \frac{R_c^2}{\sqnorm{u}}\right] \frac{\sqnorm{u}}{d} e^{-\frac{\sqnorm{u}}{2}}{\rm d} u,
    \end{align*}
    where the last line follows from the fact that the expression within the integral is completely symmetric, so for any $u$, each direction receives a weight proportional to $d^{-1}$. In particular, $\esp{\tilde{A} - \reg \Id} = \bar{\rho} \Sigma$ with
    \begin{equation}
        \bar{\rho} = \frac{1}{d(2\pi)^{d/2}}\int_{\R^d} \min(\sqnorm{u}, R_c^2) e^{-\frac{\sqnorm{u}}{2}}{\rm d} u,
    \end{equation}
    which only depends on $R_c$, and is close to $1$ for instance when $R_c^2 \geq 2d$. We know however that $\bar{\rho} \leq 1$ since it comes from projecting points from a Gaussian distribution of covariance $\Sigma$. 
    
    \textbf{3 - PTR always succeeds with this choice of $\rho, \epsilon, \delta$.} We have shown that with high probability, the regularized clipped covariance matrix is proportional (up to a factor $\bar{\rho}\leq 1$) to the covariance of the data. We now show that in this case, suggesting such a $\rho$ (which we can evaluate) will pass the PTR test for a range of $(\epsilon, \delta)$ values. 

    Let $\beta > 0$. If we propose value $\rho = \beta \bar{\rho}$ then $\tilde{A} - \rho C \succcurlyeq (1 - \beta) \Sigma$. In particular, using the above derivations and clipping we obtain that for all $i$, 
    \begin{equation}
        \Xtilde_i^\top (\tilde{A} - \rho C)^{-1} \Xtilde_i \leq (1 - \beta)^{-1} \Xtilde_i^\top \Sigma^{-1} \Xtilde_i \leq (1 - \beta)^{-1} R_c^2.  
    \end{equation}
    Since each term in the summation can be lower bounded, we obtain that 
    \begin{equation}
        \min \{|I_R|, \sum_{i \in I_R} \frac{1}{n}\Xtilde_i^\top (\tilde{A} - \rho C)^{-1} \Xtilde_i \geq 1\} \geq \min \{|I_R|, |I_R| \frac{R_c^2}{n(1 - \beta)} \geq 1\} = \left\lceil\frac{n(1 - \beta)}{R_c^2}\right\rceil
    \end{equation}
    In particular, if $n \geq - \frac{\log(\delta) R_c^2}{(1 - \beta)\varepsilon}$, we have that
    \begin{equation}
        \hat{\Delta} = \Delta_+ + Lap(\varepsilon^{-1}) \geq \Delta_+ \geq n(1 - \beta) / R_c^2 \geq - \log(\delta) / \varepsilon,
    \end{equation}
    and so the PTR always returns a value. Take $\beta = \frac{1}{2}$ to instantiate the theorem. 

    Note that Proposition~\ref{prop:ellipse_gaussian} uses the condition number of the covariance $\kappa_\Sigma$ instead of $\kappa$, the condition number of $A$. This is because $\kappa$ depends on the specific dataset that we consider, and so we should also run a PTR procedure to release it. Note that in the end, we pay a factor $\kappa_\Sigma$ since:
    \begin{equation}
        \sqnorm{X_i} X_i^\top A^{-2} X_i \leq \rho^{-2} \sqnorm{X_i} X_i^\top \Sigma^{-2} X_i \leq \kappa_\Sigma \rho^{-2} ( X_i^\top \Sigma^{-1} X_i)^2.
    \end{equation}
    
\end{proof}

\section{Details on the experimental setup}
\label{app:experiments}

In this section, we provide the various details needed to reproduce our results. Note that we also provide code in supplementary material. Our experimental setup assumes that the data is split across 2 nodes. At each step $t$, node $i$ privately releases its gradient $g_t^{(i)}$, computed at point $\theta_t$. For all methods, we run the following gradient descent algorithm: 
\begin{equation}
    \theta_{t+1} = \theta_t - \tau \times \frac{1}{2}(g_t^{(1)} + g_t^{(2)}),
\end{equation}
where $\tau$ is the step-size. In order to comply with the guidelines of Theorem~\ref{thm:utility}, we set it as half the maximum of the largest eigenvalue of the local covariance matrices, so $\tau = 0.5 / \max (\lambda_{\max}(A^{(1)}), \lambda_{\max}(A^{(2)}))$ which is an approximation for the true largest possible step-size. This is the step-size we use regardless of the method used for noising. Then, all methods only differ in the way they add noise to the gradients. The non-private method thus releases: 
\begin{equation}
    g_t^{(i)} = \frac{1}{N} \sum_{j=1}^N \nabla f_{ij}(\theta) = \frac{1}{N} \sum_{j=1}^N A_{ij} \theta - b_{ij}.
\end{equation}
    
\subsection{Parameters for GD with clipping}
For gradient clipping, each node clips all its individual gradients using a clipping threshold $c^{(i)}$. In particular, the noisy gradients write: 
\begin{equation}
    g_t^{(i)} = \frac{1}{N} \sum_{j=1}^N \frac{\nabla f_{ij}(\theta_t) c_i}{\max(c_i, \norm{\nabla f_{ij}(\theta_t)})} + \mathcal{N}(0, \sigma_i^2),
\end{equation}
where the variance of the noise $\sigma_i^2$ is computed as 
\begin{equation}
    \sigma_i^2 = \frac{\alpha c_i^2}{\varepsilon N^2},
\end{equation}
where $\alpha$ and $\varepsilon$ are the Rényi-DP parameters. The results of DP-GD heavily depend on how we set the clipping threshold $c_i$. 

One way of setting it is simply to randomly try out some clipping thresholds (potentially with external knowledge), but this is quite inefficient, as each node node needss to release $M$ times more gradients to the other node if we would like to try $M$ thresholds. Instead, we assume in this work that we can estimate the stochastic gradients at the optimum for the function we consider, and choose the smallest threshold such that none of these gradients is clipped, so $c_i = \max_j \norm{\nabla f_{ij}(\theta_i^\star)}$. This leverages auxiliary information, and is a very strong baseline in the homogeneous setting, but leads to small clipping thresholds in the heterogeneous setting. Small clipping thresholds induce small noise but potentially large bias. We observe in our experiments that thresholds significantly lower than $c_i$ do not introduce such a large bias still. We conjecture that this is due to the fact that the bias introduced from clipping is small for specific distributions, such as symmetric ones. Also note that per-instance clipping is in general computationally expensive, and in particular required significantly more time in our case.

\subsection{Parameters for GD with relative sensitivity}
In this case, the noisy gradients we release are of the form: 
\begin{equation}
    g_t^{(i)} = \bar{g_t}^{(i)} + \mathcal{N}(0, \gamma_i \|\bar{g_t}^{(i)}\|^2 + \sigma_i^2), \text{ where } \bar{g_t}^{(i)} = \frac{1}{N} \sum_{j=1}^N \tilde{A}_{ij} \theta_t - \tilde{b}_{ij}.
\end{equation}
The parameters $\gamma_i$ and $\sigma_i^2$ depend on the relative sensitivity parameters $\simi$ and $\Rrel$, as specified in Theorem~\ref{thm:renyi-sgm}. We set $\eta$ directly using Equation~\eqref{eq:condition_Lrel}, which we can check directly on our dataset. We could alternatively recover this threshold by privately releasing a rough approximation of the covariance matrix, and then setting $R_c$ and $\rho$ by binary search (or using a small public dataset) to obtain a small $\eta$ at the cost of logarithmic factors only (multiplying only the PTR procedure, not the whole training).

To set $\Rrel$, we use that it is bounded by $\simi\norm{b_i} + 2\max{\norm{b_{ij}}}$. Thus, we bound it by computing the max norm of the $b_{ij}$, which can also be enforced via clipping. Although $\norm{b_i} \leq \max\norm{b_{ij}}$, we use a separate threshold for $\norm{b_i}$ which usually leads to tighter guarantees. Note that we can alternatively approximate the local optimum $\theta_i^\star$ and set $\Rrel = \max_{k, \ell} \norm{\nabla f_{ik}(\theta_i^\star) - \nabla f_{i\ell}(\theta_i^\star)}$. However, we choose not to in this case because we want something that is \emph{enforceable}, and this bound is harder to enforce via clipping the $b_{ij}$. The privacy loss term $\alpha \eta^2 d$ is negligible in our case since datasets have moderate dimension. As the examples are mainly intended to be illustrative and the constants have not been optimized for, we omit constant factors when estimating $\simi$.

As in the clipping case, this specific procedure of choosing the hyperparameters does not strictly guarantee differential privacy, as it uses local datasets. Yet, in practice, application-specific knowledge can help set these parameters. We choose these favorable parameters to show what different methods can do with appropriate hyperparameters, without explicitly quantifying the privacy loss incurred by having to find these parameters. Besides, \textit{we choose parameters that can be enforced}, so that strict DP is guaranteed if we initially have a rough idea of what parameters should be. 

We present experiments on the \textit{ijcnn1} dataset (concatenation of train and test from LibSVM repository\footnote{\url{https://www.csie.ntu.edu.tw/~cjlin/libsvmtools/datasets/binary.html}}) since the estimated $\simi$ were rather small, and so could illustrate a case in which $\RGM$ is useful. This is also the case for other LibSVM datasets that we tested, such as \textit{HIGGS} or \textit{SUZY}. Other datasets such as \textit{cod-rna} require high levels of regularization to obtain small $\simi$, potentially trivializing the initial problem.

\section{Extra experiments}
\label{app:extra_exp}
To show that our results apply more broadly, we include results on the Higgs (truncated at $n=10^6$, $\varepsilon=10^{-2}$), in the same setting as Figure~\ref{fig:exp} (left). We see that similar results to the ijcnn1 datasets are obtained. To compare further, we also try a high privacy regime for ijcnn1 ($\varepsilon=10^{-3}$), in which we see that the performance of RGM is close to that of clipping with the best (small) threshold. 

\begin{figure}
    \centering
    \includegraphics[width=0.3\linewidth]{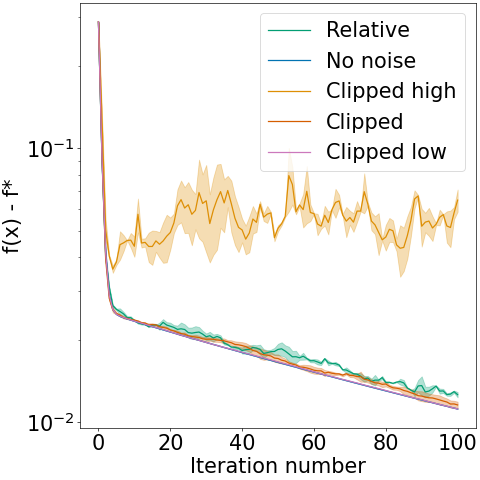}
    \includegraphics[width=0.3\linewidth]{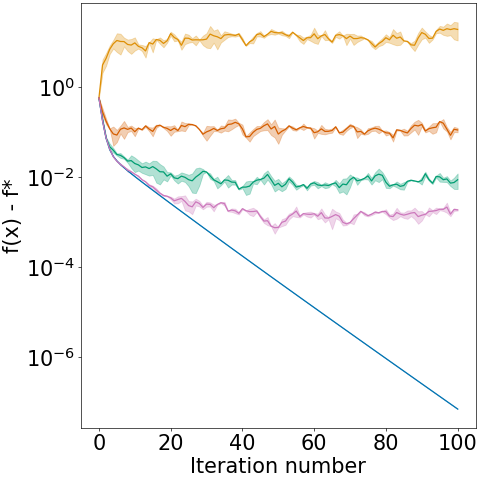}
    \caption{Left: HIGGS dataset, $\varepsilon=10^{-2}$, Right: ijcnn1 dataset, $\varepsilon=10^{-3}$}
    \label{fig:exp_app}
\end{figure} 

Note that although the experiments are for linear regression only, the relative sensitivity bounds can be obtained for more general functions, as detailed in Appendix B.2. Such relative smoothness assumptions can for instance be shown for generalized linear models. Similarly, our theory can handle non-convex functions, as long as the gradients satisfy the relative sensitivity assumption. The difficulty of applying our approach to neural networks is not related to non-convexity per se, but to the complexity of the functions considered in these cases, preventing meaningful relative sensitivity bounds.

\end{document}